\documentclass[twoside]{article}

\usepackage{amssymb}
\usepackage{enumitem}
\usepackage{anyfontsize}
\usepackage{amsmath, amsthm, amssymb, graphicx, color, hyperref}

\theoremstyle{plain}
\newtheorem{thm}{Theorem}
\newtheorem{lem}{Lemma}
\newtheorem{prop}{Proposition}

\theoremstyle{definition}

\theoremstyle{remark}

\DeclareMathOperator{\sgn}{sgn}

\usepackage[accepted]{aistats2025}
\begin{document}

%

%

\twocolumn[
\runningtitle{Optimizing Noise Schedules of Generative Models in High Dimensions}
\aistatstitle{Optimizing Noise Schedules of Generative Models\\ in High Dimensions}
\aistatsauthor{ Santiago Aranguri \And Giulio Biroli \And  Marc Mézard \And Eric Vanden-Eijnden }
{\fontsize{9.6pt}{12pt}\selectfont
\aistatsaddress{ Courant Institute of\\ Mathematical Sciences,\\New York University \And  Laboratoire de Physique de\\l’Ecole Normale Supérieure\\ENS, Université PSL\And Department of Computing \\Sciences, Bocconi University \And Courant Institute of\\ Mathematical Sciences,\\New York University }}
]

\begin{abstract}
Recent works have shown that diffusion models can undergo phase transitions, the resolution of which is needed for accurately generating samples. This has motivated the use of different noise schedules, the two most common choices being referred to as variance preserving (VP) and variance exploding (VE). Here we revisit these schedules within the framework of stochastic interpolants. Using the Gaussian Mixture (GM) and Curie-Weiss (CW) data distributions as test case models, we first investigate the effect of the variance of the initial noise distribution and show that VP recovers the low-level feature (the distribution of each mode) but misses the high-level feature (the asymmetry between modes), whereas VE performs oppositely. We also show that this dichotomy, which happens when denoising by a constant amount in each step, can be avoided by using noise schedules specific to VP and VE that allow for the recovery of both high- and low-level features. Finally we show that these schedules yield generative models for the GM and CW model whose probability flow ODE can be discretized using $\Theta_d(1)$ steps in dimension $d$ instead of the $\Theta_d(\sqrt{d})$ steps required by constant denoising.
\end{abstract}

\section{Introduction}
Generative models based on dynamical transport of measure have emerged as powerful tools in unsupervised learning  (\cite{sohl2015deep}, \cite{ho2020denoising}, \cite{song2021scorebasedgenerativemodelingstochastic}). These models are trained on samples from a data distribution and tasked with generating new samples from this distribution. This can for example be achieved with probability flow ODEs, in which  the  data samples are used to learn a velocity field that pushes samples from a simple distribution (say a Gaussian) onto new samples from the data distribution. Probability flow ODEs can achieve state of the art performance in  image generation \cite{esser2024scalingrectifiedflowtransformers}, but  they are not fully understood theoretically. A theoretical framework not only would suggest improvements, but also greatly reduce the architecture and hyperparameter search that practitioners must do to get accurate results.

The main sources of error for probability flow ODE are (a) estimating the velocity field for the generative model and (b) running a discretization of this velocity field. In this work, we assume access to the exact velocity field and analyze the error from running the generative model with a discretized version of the velocity field, as the dimension of the input data goes to infinity.

To this end, we build upon the work of \cite{biroli2024dynamicalregimesdiffusionmodels}, where score-based diffusion models are studied and shown to exhibit a phase transition in the generative process. This transition occurs at a time that goes to infinity as the dimension of the data grows, thereby requiring to run the diffusion model for longer times as the dimension increases to capture the transition. Here we revisit these results within the framework of Stochastic Interpolants \cite{albergo2023buildingnormalizingflowsstochastic,albergo2023stochasticinterpolantsunifyingframework} (also see \cite{lipman2023flow} and \cite{liu2022flowstraightfastlearning}), where the generative process happens in $[0, 1]$ instead of $[0, \infty)$ which is the case in \cite{biroli2024dynamicalregimesdiffusionmodels}, and is performed via solving a probability flow ODE. With a uniform noise schedule, we show the solutions to this ODE display a speciation transition a time that goes to zero as the dimension tends to infinity and therefore cannot be resolved by uniform time-discretization of the ODE. However, we also show  that we can select an appropriate noise schedule to ensure that this speciation transition is resolved.

Specifically, we analyze the probability flow ODE generating samples from a two-mode Gaussian mixture.
\begin{itemize}[leftmargin=0.15in]
    \item We shed light into the Variance Preserving (VP) and Variance Exploding (VE) regimes used by practitioners by showing that with a uniform noise schedule, VP captures only the high-level feature whereas VE captures only the low-level feature. We show that with the appropriate noise schedules, VP and VE can capture both phases.
    \item We find that the velocity field simplifies for each phase, reducing to the estimation of the high-level feature in the beginning and then sharply changing to the estimation of the low-level feature.
    \item We show that the length of the two phases depends on the initial variance and the time-schedule of the interpolation, which need to be chosen correctly for accurate estimation of both phases. Not doing so leads to either of the phases having length that vanishes as $d\to\infty,$ implying that the generative process does not capture some of the features from the data. 
    \item In fact, with the correct initial variance and time schedule, the generative process has a well-defined limiting ODE as $d\to\infty$. In particular, this means that we get accurate estimation of the data discretizing the generative process using $\Theta_d(1)$ points\footnote{We recall that $f(d)=\Theta_d(g(d))$ if $\lim_{d\to\infty} f(d)/g(d) \in (0,\infty).$}. If we discretize with a uniform grid, we require $\Theta(\sqrt{d})$ points for accurate estimation.
\end{itemize}

\section{Related work}
\cite{bm} consider the Curie-Weiss model as the data distribution, and introduced the idea of an speciation time, being the time after which the generative process has found the high-level structure of the data. They further find the dependence of the speciation time on the dimension of the data. Later \cite{biroli2024dynamicalregimesdiffusionmodels} generalize the definition of speciation time, which can be described in terms of cloning trajectories \cite{biroli2024dynamicalregimesdiffusionmodels} which consists of noising by a certain amount and defining the speciation transition as the amount of noise that is needed to make trajectories with independent Brownian motion terms started at the noisy datapoint speciate (i.e. go to different classes.) See also the related U-Turn method from \cite{behjoo2024uturndiffusion} and \cite{sclocchi2024phasetransitiondiffusionmodels}. The present paper builds on their work, and focuses on understanding how the speciation time is affected by time-dilations and the initial variance of the generative process. The correct choice of these leads to a well-defined limiting probability flow ODE as $d\to\infty$ which allows us to characterize the generative model in the different phases. 

\cite{raya2023spontaneoussymmetrybreakinggenerative} show the existence of spontaneous symmetry breaking, dividing the generative dynamics into two phases: 1) a linear steady-state dynamics around a central fixed-point, 2) an attractor dynamics towards the data manifold. They provide a theoretical analysis supporting their claim for the empirical distribution on $k$ undistinguishable samples and for the hypersphere. In the present paper, however, through considering asymptotics and time-dilations, we are able to analyze more complex data distributions, and give exact expressions for the velocity field of the generative model. This allows us to show how each feature of the data is learned, and show that the velocity field simplifies in each phase.

\cite{li2024criticalwindowsnonasymptotictheory} study diffusion models through critical windows, showing that for strongly log-concave data, one can associate a narrow time window to the learning of a given feature of the data. Further, they interpret diffusion models as hierarchical samplers that progressively decide output features. The main differences with our work is that they run the generative model in continuous time, whereas our analysis answers the important practical question of how many discretization steps are needed to capture different features. Our analysis allows us to give precise characterizations of the dynamics in each phase. We can also handle non-log-concave distributions, and we do so for the Curie-Weiss distribution, showing that the separation of phases extends beyond log-concave distributions. As opposed to ours, their analysis is non-asymptotic.

Several works focus on providing bounds for the number of required steps for generative models to accurately generate data. They either have general mild assumptions on the smoothness of the data (\cite{NEURIPS2023_d84a27ff}, \cite{benton2024nearlydlinearconvergencebounds}) or closer to our work, assume Gaussian Mixture data
(\cite{gatmiry2024learningmixturesgaussiansusing}, \cite{shah2023learningmixturesgaussiansusing}). They require a polynomial in dimension number of steps for accurate estimation, with the best bound being for the probability flow ODE requiring $O(\sqrt{d})$ steps. In the present work we assume a simpler distribution (a two-mode Gaussian Mixture) but we are able to show that if one uses the right noise schedule, $\Theta_d(1)$ steps suffice for accurate estimation of this distribution.

\section{Results}
\subsection{Set up}
We consider the Gaussian Mixture (GM) distribution
\begin{align*}
    \mu = p\mathcal{N}(r,\sigma^2\text{Id}) + (1-p)\mathcal{N}(-r,\sigma^2\text{Id})
\end{align*}
where $p\in [0,1], r\in\mathbb{R}^{d}$ such that $|r|^{2}=d$ and $\sigma^2=\Theta_d(1).$

We will benchmark different settings for the probability flow ODE generating samples from the GM distribution testing whether it captures the high-level parameter $p$ and low-level parameter $\sigma^2.$ More precisely, we consider the stochastic interpolant \cite{albergo2023stochasticinterpolantsunifyingframework} 
\begin{equation*}
    I_{\tau}=c\alpha_\tau z+\beta_\tau a
\end{equation*}
where  $c>0$, $\alpha_\tau,\beta_\tau \in C^2([0,1]),$ $\alpha_0=1=\beta_1,$ $\alpha_1=0=\beta_0,$  $\dot\alpha_\tau<0 <\dot \beta_\tau,$ $z\sim{\mathcal{N}}(0,\text{Id}_{d})$ and $a\sim\mu.$ The choices of $\alpha_t,\beta_t$ imply that $I_0=cz\sim {\mathcal{N}}(0,c^2\text{Id}_{d})$ and $I_1=a\sim\mu$, that is, $I_t$ interpolates between noise and data. We assume that $\alpha_\tau, \beta_\tau$ do not depend on~$d$, and will use $c$ to scale the noise with dimension. Borrowing terminology from \cite{song2021scorebasedgenerativemodelingstochastic}, we will call the stochastic interpolants with $c=1$ the \textbf{Variance Preserving} (VP) interpolants and $c=\sqrt{d}$ the \textbf{Variance Exploding} (VE) interpolants.

Associated with any interpolant there is a generative model that pushes through an ODE or an SDE to generate samples from the data distribution. In this setup, \cite{albergo2023stochasticinterpolantsunifyingframework} prove the following result.

\begin{lem}
    Let $X_{\tau}$ solve the probability flow ODE  
    \begin{equation}
    \Dot{X}_{\tau}=b_{\tau}(X_{\tau})\qquad\text{with}\qquad b_{\tau}(x)=\mathbb{E}[\Dot{I}_{\tau}|I_{\tau}=x]\label{eq:ode}
    \end{equation}
    If $X_{0}\sim \mathcal{N}(0,c^2\text{Id}_{d})$, then $X_{\tau}\stackrel{d}{=}I_{\tau}$ for all $\tau\in[0,1]$, and in particular, $X_{\tau=1}\sim\mu$. 
\end{lem}

In practice, the ODE from \eqref{eq:ode} needs to be discretized. We are interested in how well this probability flow ODE generates samples from the data, when the number of discretization points does not grow with $d.$ In the next proposition, we consider a VP interpolant, and show that if we discretize with a uniform grid whose step size is $\Theta_d(1)$ we will not recover the parameter $p.$ However, the VP interpolants are able to recover the distribution within each mode (given by the parameter $\sigma^2$).

\begin{prop}[VP only captures $\sigma^2$] Let $X^{\Delta \tau}_\tau$ be obtained from the probability flow ODE \eqref{eq:ode} associated with a VP interpolant discretized with a uniform grid with step size $\Delta \tau.$ Let 
\label{prp:vp}
\begin{align*}
    M_\tau^{\Delta \tau, d}=r\cdot X_\tau^{\Delta \tau}/d, \quad  X_\tau^{\perp} = X^{\Delta \tau}_\tau - M_\tau^{\Delta \tau, d}r.
\end{align*} Then
\begin{align*}
    \lim_{\Delta \tau\to 0}\lim_{d\to\infty} M^{\Delta \tau, d}_1 \stackrel{d}{=} \hat p \delta_1 + (1-\hat p) \delta_{-1}
\end{align*}
where
\begin{align*}
    \hat p =
    \begin{cases}
        1 \qquad \text{ if } p>1/2\\
        1/2 \quad \text{ if } p=1/2\\
        0 \qquad \text{ if } p<1/2\
    \end{cases}
\end{align*}
Morevover, for $\tau=\ell\Delta \tau\in [0,1]$, $\ell\in\mathbb N,$ we have
\begin{align*}
    X^{\perp}_\tau \sim \mathcal{N}(0,\left( \sigma^{\Delta \tau}_\tau\right)^2\text{Id}_{d-1}).
\end{align*}
where $\sigma^{\Delta \tau}_\tau$ satisfies
\begin{align*}
    \lim_{\Delta \tau\to 0} \sigma^{\Delta \tau}_1 = \sigma.
\end{align*}
\end{prop}

This proposition can be better understood by looking at the speciation time $\tau_s,$ which is the time in the generative process after which it is determined what mode the sample will belong to. \cite{biroli2024dynamicalregimesdiffusionmodels} determine the speciation time for the GM distribution, working with score-based diffusion models on $[0,\infty).$ We show in appendix, Proposition \ref{prp:spec:vp}, that translating their result to the context of stochastic interpolants gives that the VP interpolants have speciation time $\tau_s$ that goes to $0$ as $d$ goes to infinity, meaning that the time when the generative model can recover the parameter $p$ vanishes as the dimension $d$ goes to infinity. This explains why the discretization with a finite $\Delta\tau$ cannot recover the correct speciation transition.

On the contrary, we prove next that the VE interpolants have speciation time $\tau _s\in (0, 1)$ as $d\to\infty$, and therefore their generative model can recover the correct value of $p$ in the limit of $d$ going to infinity when discretized with a uniform grid with $\Theta_d(1)$ steps. However, as this generative model with VE interpolants focuses on the speciation transition, it fails to reconstruct the right probability distribution of each mode: it cannot obtain the right value of $\sigma^2$  in the $d\to \infty$ limit.

\begin{prop}[VE only captures $p$] 
\label{prp:ve}
 Let $X^{\Delta t}_t$ be obtained from the probability flow ODE \eqref{eq:ode} associated with the VE interpolant discretized with a uniform grid with step size $\Delta \tau.$ Let \begin{align*}
    M_\tau^{\Delta \tau, d}=r\cdot X_\tau^{\Delta \tau}/d, \quad  X_\tau^{\perp} = X^{\Delta \tau}_\tau - M_\tau^{\Delta \tau, d}r.
\end{align*} 
Then $M_\tau=\lim_{\Delta \tau\to 0}\lim_{d\to\infty} M^{\Delta \tau, d}_\tau$ fulfills the ODE
\begin{align*}
    \dot M_\tau = \frac{\dot\alpha_\tau}{\alpha_\tau}M_t+ \frac{\alpha_\tau\dot\beta_\tau-\dot\alpha_\tau\beta_\tau}{\alpha_\tau}\tanh\left(h+\frac{\beta_\tau M_{\tau}}{\alpha^2_\tau}\right)
\end{align*} 
where $h$ is such that $e^{h}/(e^{h} + e^{-h}) = p$. This yields a speciation time $\tau_s\in (0,1)$ independent of $d$, and implies that 
\begin{align*}
    \lim_{\Delta \tau\to 0}\lim_{d\to\infty} M^{\Delta \tau, d}_1 \stackrel{d}{=} p \delta_1 + (1- p) \delta_{-1}
\end{align*}
Moreover, for $\tau=\ell\Delta \tau \in [0, 1]$, $\ell\in\mathbb{N},$ we have
\begin{align*}
    X^{\perp}_\tau \sim \mathcal{N}(0,\left( \sigma^{\Delta \tau, d}_\tau\right)^2\text{Id}_{d-1}).
\end{align*}
where $\sigma^{\Delta \tau, d}_\tau$ satisfies
\begin{align*}
    \lim_{\Delta \tau\to 0}\lim_{d\to\infty} \sigma^{\Delta \tau, d}_1 = 0.
\end{align*}
\label{prp:van:ve}
\end{prop}

\textbf{Dichotomy between VP and VE interpolants.} The discussion so far shows a dichotomy: the $p$ is captured by the VE but not the VP interpolant, whereas the $\sigma^2$ is captured by the VP but not the VE interpolant. Next we show how to introduce time dilations to solve this issue. 

\subsection{Time dilated interpolants}
For concretness, we specialize to $\alpha_\tau=1-\tau$ and $\beta_\tau=\tau.$ We define the \textit{Dilated Variance Preserving interpolant}
\begin{align}
    I^\text{P}_{t} := I^\text{P}_{\tau=\tau_t}=(1-\tau_t)z+\tau_t a
\end{align}
and the \textit{Dilated Variance Exploding interpolant}
\begin{align}
    I^\text{E}_{t} := I^\text{E}_{\tau=\tau_t}=\sqrt{d}(1-\tau_t)z+\tau_t a.
\end{align}
If we think of running a discretization of the probability flow ODE associated to any of these interpolants, we can see each step in the discretization as denoising by a certain amount. Introducing a time dilation, then, is equivalent to using a non-uniform noise schedule. Next, we give specific time dilations $\tau_t$ for the VP and VE interpolants, which will enable us to capture both $p$ and $\sigma^2.$ 

\textbf{Dilated Variance Preserving.} 
We saw earlier that the VP interpolants have a speciation times that go to $0$ as $d$ goes to infinity. In fact, if $\alpha_t=1-t,\beta_t=t$ we prove in the appendix, Proposition \ref{prp:spec:vp}, that the VP interpolant has speciation time $1/\sqrt{d}.$ This motivates a time dilation that makes the speciation time happen at a constant time $\tau_s\in (0, 1)$ as $d\to\infty.$ To this end, we consider
\begin{align}
    \label{eq:vp:time_dil}
    \tau(t) = \begin{cases}
        \frac{2\kappa t}{\sqrt{d}} & \text{if } t \in [0,\frac{1}{2}]\\
        \frac{\kappa}{\sqrt{d}} + \left(1-\frac{\kappa}{\sqrt{d}}\right)(2t-1) & \text{if } t \in [\frac{1}{2}, 1]
    \end{cases}
\end{align}
which satisfies $\tau(0)=0,$ $\tau(\frac12)=\kappa/\sqrt{d},$ $\tau(1)=1$ where $\kappa$ is a constant. 

We now prove that with this time dilation, the generative model associated to the VP interpolant has a well-defined limiting ODE. This implies that the dilated VP captures both $p$ and $\sigma^2.$

\begin{thm}[Dilated VP captures $p$ and $\sigma^2$]
    Let $X^{\Delta t}_t$ be obtained from the probability flow ODE associated with the dilated VP interpolant discretized with a uniform grid with step size $\Delta t.$ Then 
    \label{thm:dvpa}
    \begin{align*}
        X^{\Delta t}_t - \frac{r\cdot X^{\Delta t}_t}{{d}}r\sim \mathcal{N}\left(0, \left( \sigma^{\Delta t, d}_t\right)^2\text{Id}_{d-1}\right).
    \end{align*}
    where $ \sigma^{\Delta t, d}_t$ is characterized as follows:\\
    \textbf{First phase}: For $t\in [0,\tfrac{1}{2}]$ we have 
    \begin{align*}
        \lim_{\Delta t\to 0}\lim_{d \to\infty}  \sigma^{\Delta t, d}_t =1. 
    \end{align*}
    In addition  
        \begin{align*}
            \mu_t = \lim_{\Delta t\to 0}\lim_{d \to\infty} \frac{r\cdot X^{\Delta t}_t}{\sqrt{d}}
        \end{align*}
        fulfills 
        \begin{align*}
            \dot \mu_t &= 2\kappa \tanh\left(h+2\kappa t \mu_t\right),\quad \mu_{t=0}\sim \mathcal{N}(0,1).
        \end{align*}
        where $h$ is such that $p=e^h/(e^h+e^{-h}).$ This implies $\mu_{t=1/2}\sim p\mathcal{N}(\kappa, 1) + (1-p)\mathcal{N}(-\kappa, 1).$\\ 
    \textbf{Second phase}: For $t\in [\tfrac{1}{2}, 1]$ we have 
    \begin{align*}
        \lim_{\Delta t\to 0}\lim_{d \to\infty}  \sigma^{\Delta t, d}_t =\sqrt{(2-2t)^2+(2t-1)^2\sigma^2}
    \end{align*}
    In addition 
        \begin{align*}
            M_t = \lim_{\Delta t\to 0}\lim_{d \to\infty}  \frac{r\cdot X^{\Delta t}_t}{{d}}
        \end{align*}
        fulfills, for $t\in (1/2, 1)$, the ODE 
        \begin{align*}
            \dot M_t &= \frac{-(1-t)+\sigma^2(t-\tfrac{1}{2})}
            {(1-t)^2+\sigma^2(t-\tfrac{1}{2})^2}M_{t} + \frac{(1-t) \sgn(M_t)}{(1-t)^2+\sigma^2(t-\tfrac{1}{2})^2}
        \end{align*}
        and satisfies
        \begin{align*}
            M_1 \sim p^\kappa\delta_1 + (1-p^\kappa)\delta_{-1}
        \end{align*}
        where $p^\kappa$ is such that  $\lim_{\kappa\to \infty}p^\kappa=p$
\end{thm}
Since we take first $d\to\infty$ and then $\Delta t\to 0,$ this means we can discretize the ODE with $\Delta t\in \Theta_d(1)$ and get accurate estimation.

We emphasize that the first phase captures the relative asymmetry $p,$ whereas the second phase captures the distribution of each mode, through the estimation of~$\sigma^2.$ This result can be equivalently seen as the fact that using a discretization grid for the non-dilated VP interpolant with half grid points on $[0,\kappa/\sqrt{d}]$ and half on $[\kappa/\sqrt{d}, 1]$ yields accurate estimation of $p$ and $\sigma^2.$ In particular, this means that using a uniform discretization with $\sqrt{d}$ grid points gives correct estimation of $p$ and $\sigma^2$ for the non-dilated VP interpolant. However, we show in the appendix, Proposition \ref{prp:spec2}, that discretizing the probability flow ODE for the VP interpolant with $o(\sqrt{d})$ uniform grid points does not yield correct estimation of $p.$

\textbf{Dilated Variance Exploding.} We saw in Proposition \ref{prp:van:ve} that the vanilla VE interpolant can not recover $\sigma^2$. This can be seen intuitively by looking at the interpolant at a given coordinate 
$$(I^E_\tau)^i = \sqrt{d}(1-\tau)z^i + \tau a^i$$
and noting that only for $\tau\in [1-\kappa/\sqrt{d}, 1]$ the noise and data have coefficients of the same magnitude. Without dilation, this window disappears as $d$ goes to infinity. We prove that dilating time so as to have this window of constant length solves the problem. More precisely, consider the time-dilation
\begin{align}
    \label{eq:t_dil:ve}
    \tau(t) = \begin{cases}
        \left(1-\frac{\kappa}{\sqrt{d}}\right)2t & \text{if } t \in [0,\frac{1}{2}]\\
        \left(1-\frac{\kappa}{\sqrt{d}}\right) + \frac{\kappa}{\sqrt{d}}(2t-1) & \text{if } t \in [\frac{1}{2}, 1]
    \end{cases}
\end{align}

\begin{thm}[Dilated VE captures $p$ and $\sigma^2$]
    \label{prp:dve}
    Let $X^{\Delta t}_t$ be obtained from the probability flow ODE associated with the dilated VE interpolant discretized with a uniform grid with step size $\Delta t.$ Let \begin{align*}
        M_\tau^{\Delta \tau, d}=r\cdot X_\tau^{\Delta \tau}/d, \quad  X_\tau^{\perp} = X^{\Delta \tau}_\tau - M_\tau^{\Delta \tau, d}r
    \end{align*}
    and let $M_t = \lim_{\Delta t\to 0}\lim_{d \to\infty} M^{\Delta t, d}_t.$\\
    \textbf{First phase}: For $t\in [0,\tfrac{1}{2}],$ we have
    \begin{align*}
        X_t^{\perp}\sim \mathcal{N}\left(0, d\left( \sigma^{\Delta t, d}_t\right)^2\text{Id}_{d-1}\right).
    \end{align*}
    where
    \begin{align*}
        \lim_{\Delta t\to 0}\lim_{d \to\infty} \sigma^{\Delta t, d}_t= 1-2t
    \end{align*}
    Moreover, $M_t$ fulfills
    \begin{align*}
        \dot M_t = \frac{-M_t + \tanh\left(h+\frac{2tM_t}{(1-2t)^2}\right)}{\tfrac{1}{2}-t}
    \end{align*}
    with $M_{t=0}\sim \mathcal{N}(0,1)$ and $h$ is such that $p=e^h/(e^h+e^{-h}).$ This implies 
    \begin{align*}
        M_{1/2} \sim p\delta_1+(1-p)\delta_{-1}.
    \end{align*}
    \textbf{Second phase}: For $t\in [\tfrac{1}{2}, 1],$ we have that $M_t=M_{t=1/2}$ remains constant. Moreover,
    \begin{align*}
        X_t^{\perp}\sim \mathcal{N}\left(0, \left( \sigma^{\kappa, \Delta t, d}_t\right)^2\text{Id}_{d-1}\right)
    \end{align*} 
    where 
    \begin{align*}
        \sigma^\kappa_t=\lim_{\Delta t\to 0}\lim_{d\to \infty} \sigma^{\kappa, \Delta t, d}_t = \kappa\sqrt{\frac{\kappa^2(2-2t)^2+\sigma^2}{\kappa^2+\sigma^2}}.
    \end{align*}
    In particular, $\lim_{\kappa\to\infty}\sigma_1^\kappa=\sigma.$
\end{thm}

Similarly to Theorem \ref{thm:dvpa}, we see that the first phase captures $p$ whereas the second phase captures $\sigma^2.$ However, we see that the role of the dilation is different in the VP and the VE interpolants. For the VP interpolant, the estimation of $\sigma^2$ is solved ``by default,'' whereas estimating $p$ requires $\kappa$ large enough. The VE interpolant performs oppositely.


\subsection{Connection with Score-based diffusion models}
\label{sec:sbdm}
\textbf{Variance Preserving SDE.} Consider the VP SDE from \cite{song2021scorebasedgenerativemodelingstochastic}
\begin{align}
    \label{eq:sbdm:vp}
    dY_s = -Y_sds + \sqrt{2}dW_s;\quad Y_{s=0}\sim \mu
\end{align}
Under this SDE, $Y_s$ converges to a standard normal as $s$ tends to $\infty.$ By learning the score of the density of $Y_s,$ we can write the associated backward SDE to \eqref{eq:sbdm:vp} or the probability flow ODE and use either as a generative model. A computation yields that the law of $Y_s$ conditioned on $Y_{0}$ is given by 
\begin{align}
    \label{eq:sde:gamma}
    Y_s \sim \mathcal{N}\left(e^{-s}Y_0,\left(1-e^{-2s}\right)\text{Id}\right)
\end{align}
for $s\in[0,\infty).$ This means that $Y_s$ is equal in law to 
\begin{align*}
    I_s = \sqrt{1-e^{-2s}}z+e^{-s}a
\end{align*}
where $z\sim\mathcal{N}(0,\text{Id})$ and $a\sim \mu.$ Under the change of variables $s(t)=-\ln t,$ we have 
\begin{align}
    \label{eq:van:vp}
    I_t := Y_{s(t)} = \sqrt{1-t^2}z+ta,
\end{align}
which is a Variance Preserving interpolant with $\alpha_t=\sqrt{1-t^2},$ $\beta_t=t.$ In practice, the following more general SDE is considered
\begin{align}
    \label{eq:sbdm:vp:gamma}
    dY_s = -\gamma_sY_sds + \sqrt{2\gamma_s}dW_s,\quad Y_{s=0}\sim \mu,
\end{align}
with $\gamma_s\geq 0.$ In this case, under the change of variables $s(t)=-\ln t,$ we get that $Y_{s(t)}$ equals in law to 
\begin{align*}
    I_t = \sqrt{1-\tau_t^2}z+\tau_t a
\end{align*}
where $\tau_t = \exp\left(-\int^{-\ln t}_0 \gamma_u du\right).$ Hence, $\gamma_s$ acts as a time dilation of the first VP interpolant \eqref{eq:van:vp}.

Practitioners often run the SDE \eqref{eq:sbdm:vp:gamma} for $s\in[0,1]$ and take $\gamma_s=\gamma_{\min} + s(\gamma_{\max} - \gamma_{\min})$ where $\gamma_{\max}$ and $\gamma_{\min}$ are determined empirically. We plot in Figure \ref{fig:dilation} this time dilation using $\gamma_{\max}=20$ and $\gamma_{\min}=0.1$ as chosen by \cite{ho2020denoising} and \cite{song2021scorebasedgenerativemodelingstochastic}.

Although our analysis for the time dilated VP interpolant is for $\alpha_\tau=1-\tau$ and $\beta_\tau=\tau,$ we show in the appendix, Theorem \ref{thm:dvp}, that the VP interpolant with $\alpha_\tau=\sqrt{1-\tau^2}$ and $\beta_\tau=\tau$ also captures both $p$ and $\sigma^2$ when using the time dilation from equation \eqref{eq:vp:time_dil}.

We see in Figure \ref{fig:dilation} that the dilation from \cite{ho2020denoising} and our dilation from equation \eqref{eq:vp:time_dil} both dilate near the beginning. This suggests that our analysis may extend to broader settings beyond the probability flow ODE for the Gaussian Mixture distribution.
\begin{figure}
    \centering
    \includegraphics[width=.492\linewidth]{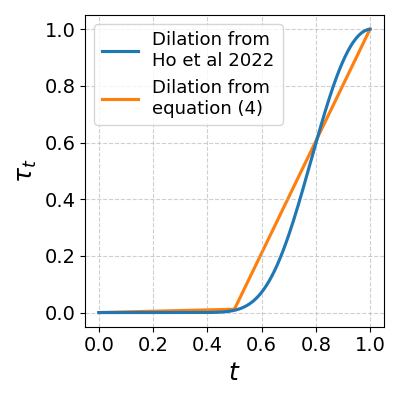}
    \includegraphics[width=.492\linewidth]{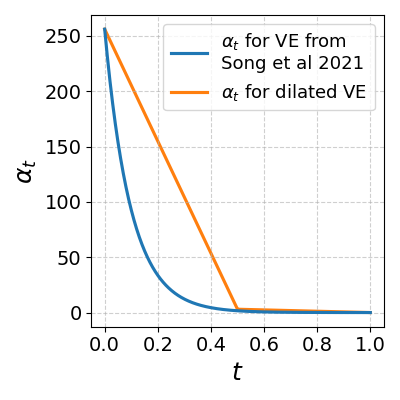}
    \caption{\textbf{(Left panel):} We consider the time dilation used by \cite{ho2020denoising} where $\tau_t = \exp\left(\gamma_{\min}\ln t -(\gamma_{\max}-\gamma_{\min})(\ln t)^2/2\right)$ for $\gamma_{\max}=20$ and $\gamma_{\min}=0.1$ and compare it with the time dilation \eqref{eq:vp:time_dil} used in our analysis, with $d=256^2$ (since \cite{ho2020denoising} works with $256\times256$ images) and $\kappa=3.$ Since the VP SDE is run til $s=1,$ the time dilation from \cite{ho2020denoising} is only used in $t\in[1/e, 1].$ \textbf{(Right panel):} We plot the magnitude $\alpha_t$ of the noise of the dilated VE interpolant, $\alpha_t=\sqrt{d}(1-\tau_t)$ with $\tau_t$ defined in \eqref{eq:t_dil:ve}. We also plot the magnitude of the noise for the VE SDE $\alpha_t=\sqrt{\sigma^2_{1-t} - \sigma^2_0}$ from \cite{song2021scorebasedgenerativemodelingstochastic}.}
    \label{fig:dilation}
\end{figure}

\textbf{Variance Exploding SDE.} The VE SDE is defined as 
\begin{align*} 
    dY_s = \sqrt{\frac{d(\sigma^2_s)}{ds}}dW_s;\quad Y_{s=0}\sim \mu
\end{align*}
where $\sigma_s = \sigma_{\min} \left( \frac{\sigma_{\max}}{\sigma_{\min}} \right)^s$ for $s\in[0, 1].$ It can be seen that $Y_s$ conditioned on $Y_{s=0}$ is given by
\begin{align*}
    Y_s \sim \mathcal{N}(Y_0,(\sigma^2_s-\sigma^2_0)\text{Id}).
\end{align*}

Hence, defining $X_t=Y_{1-t}$ we get
\begin{align}
    \label{eq:ve:sde}
    X_t \stackrel{d}{=} \sqrt{\sigma^2_{1-t} - \sigma^2_0}z + a
\end{align}
where $z\sim \mathcal{N}(0,\text{Id})$ and $a\sim \mu.$ Taking $\sigma_{\max}\gg \sigma_{\min}$ gives $X_0=\sigma_{\max}z.$ In practice, $\sigma_{\min}$ is taken to be a small constant, usually $0.01,$ while $\sigma_{\max}$ is taken to be the maximum Euclidean distance between any pair of sample from the dataset \cite{improved}. For the GM distribution, this distance is of order $\sigma_{\max} = \sqrt{d}.$ This means that the VE SDE from \cite{song2021scorebasedgenerativemodelingstochastic} and our VE interpolant both start from a noise distribution with variance $d \text{Id}.$ 

In Figure \ref{fig:dilation} we plot the magntiude $\alpha_t$ of the noise in the interpolation. For our dilated VE interpolant, this is $\alpha_t=\sqrt{d}(1-\tau_t)$ where $\tau_t$ is given in \eqref{eq:t_dil:ve}. For the $X_t$ coming from the VE SDE, we get from \eqref{eq:ve:sde} that $\alpha_t=\sqrt{\sigma^2_{1-t} - \sigma^2_0}.$ We note that our time dilation makes $\alpha_t=\Theta_d(1)$ for $t\in[1/2,1]$ and a similar behavior is achieved using the noise magnitude from the VE SDE.

\subsection{Curie-Weiss distribution}
We let $\rho(\eta)=p\delta_{m}(\eta) + (1-p)\delta_{-m}(\eta)$ and let 
$$
\rho(a|\eta) = \frac{1}{Z}\prod_i^d e^{\beta \eta a_i}
$$
where $a_i\in \{\pm 1\}$ are $d$ Ising spins and $\beta$ is inverse temperature chosen such that $m=\tanh(\beta m).$ We define the Curie-Weiss (CW) distribution as $\rho(a) = \rho(a|\eta)\rho(\eta).$

For $a\sim \rho,$ its magnetization is defined as $(1/d)\sum_i a_i.$ We note that in the $d$ limit, this magnetization is the same (up to a factor of $m$) as the magnetization $r\cdot a/{d}$ with $a\sim \mu.$ Hence if we pick $r=(1,\cdots,1),$ both the GM and CW distributions are alike in that they consist of two modes with probability $p$ and $1-p.$ However, they differ in the details: in particular, one is $\{\pm 1\}^d$-valued while the other is $\mathbb{R}^d$-valued. 

The same proof technique from Proposition \ref{prp:vp} and \ref{prp:ve} can be used to show that for the CW distribution without time dilating, the VP interpolant can not capture the parameter $p$ and the VE interpolant can not capture the distribution of the spins. However, we show next that we can solve this problem by using the same dilation for the VE interpolant from equation \eqref{eq:t_dil:ve} that we used to capture $p$ and $\sigma^2$ for the GM distribution. (See the appendix, Theorem \ref{prp:char:cw}, for an analogous claim for the dilated VP interpolant.)

\begin{thm}[Dilated VE captures both features for CW]
    \label{thm:cw:ve}
    Let $X^{\Delta t}_t$ be obtained from the probability flow ODE associated with the dilated VE interpolant for the CW distribution discretized with a uniform grid with step size $\Delta t.$ Let $r=(1,\cdots,1)$ and 
    \begin{align*}
        M_\tau^{\kappa, \Delta \tau, d}=r\cdot X_\tau^{\kappa, \Delta \tau, d}/d
    \end{align*}
    and let $M_t = \lim_{\kappa\to\infty}\lim_{\Delta t\to 0}\lim_{d \to\infty} M^{\Delta t, d}_t.$\\
    \textbf{First phase}: For $t\in [0,\tfrac{1}{2}],$ we have that $M_t$ fulfills
    \begin{align*}
        \dot M_t = \frac{-M_t + m\tanh\left(mh+\frac{2tmM_t}{(1-2t)^2}\right)}{\tfrac{1}{2}-t}
    \end{align*}
    with $M_{t=0}\sim \mathcal{N}(0,1)$ and $h$ is such that $p=e^{mh}/(e^{mh}+e^{-mh}).$ In particular, 
    \begin{align*}
        M_{1/2} \sim p\delta_1+(1-p)\delta_{-1}.
    \end{align*}
    In addition, for $w\perp r,$ $|w|=1,$ we have for $t\in[0,1/2]$
    \begin{align*}
        \lim_{\kappa\to\infty}\lim_{\Delta t\to 0}\lim_{d \to\infty} \tfrac{1}{\sqrt{d}}w\cdot (X_t-(1-2t)X_{0}) = 0
    \end{align*}
    \textbf{Second phase}: For $t\in [\tfrac{1}{2}, 1],$ we have that $M_t=M_{t=1/2}$ remains constant. Moreover, for any coordinate $i$ we have that
    \begin{align*}
        X^i_t = \lim_{\Delta t\to 0}\lim_{d\to\infty}(X^{\Delta t, d}_t)^i
    \end{align*}
    satisfies the ODE for $t\in [1/2, 1)$
    \begin{align*}
        \dot X_t^i = \frac{-X_t^i+\tanh\left(\beta m + \frac{X^i_t}{\kappa^2(2-2t)^2}\right)}{1-t}
    \end{align*}
    with the initial condition $X^i_{1/2}=\kappa X^1_{0}+m\sgn(M_{1/2}).$ This equation implies that 
    \begin{align*}
        \lim_{\kappa\to\infty} X_1^i \sim \left(\tfrac{1+m}{2}\right)\delta_1+\left(\tfrac{1-m}{2}\right)\delta_{-1}.
    \end{align*}
\end{thm}

We note that, after taking the appropiate limits, the first phase for GM and CW are identical (up to the factor of $m$, which would appear in the ODEs for GM if we had taken $r$ such that $|r|=m\sqrt{d}.$) This shows that in the first phase, the low-level differences between the GM and the CW model are not seen. It is only in the second phase that the probability flow ODE specializes to capture either the GM or CW model.

\section{Experiments}
\subsection{Numerical simulations for GM and CW models}
To confirm our results numerically, we run a discretized version of the probability flow ODE from equation \eqref{eq:ode} associated with the dilated VE interpolant. We work with dimension $d=10^6$ and discretization step $\Delta t=0.01.$ Note the number of steps is significantly smaller than $\sqrt{d}.$ In Figure \ref{fig:gm:dve}, we plot the magnetization $M_t=r\cdot X_t/d$ for different realizations of $X_t$. The plot confirms that $M_t$ gets determined in the first phase $t\in [0,1/2]$ and remains fixed for the second phase $t\in[1/2, 1].$ We also plot in Figure \ref{fig:gm:dve} the coordinates of a single realization $X_t$ for the probability flow ODE associated with the dilated VE interpolant for both the GM and CW distribution, with the same realization of initial condition $X_{t=0}$. We see that at $t=1/2$ the distributions look alike, but for $t$ close to $1$ they specialize to get samples from either the GM or CW distributions.
\begin{figure}
    \centering
    \includegraphics[width=.75\linewidth]{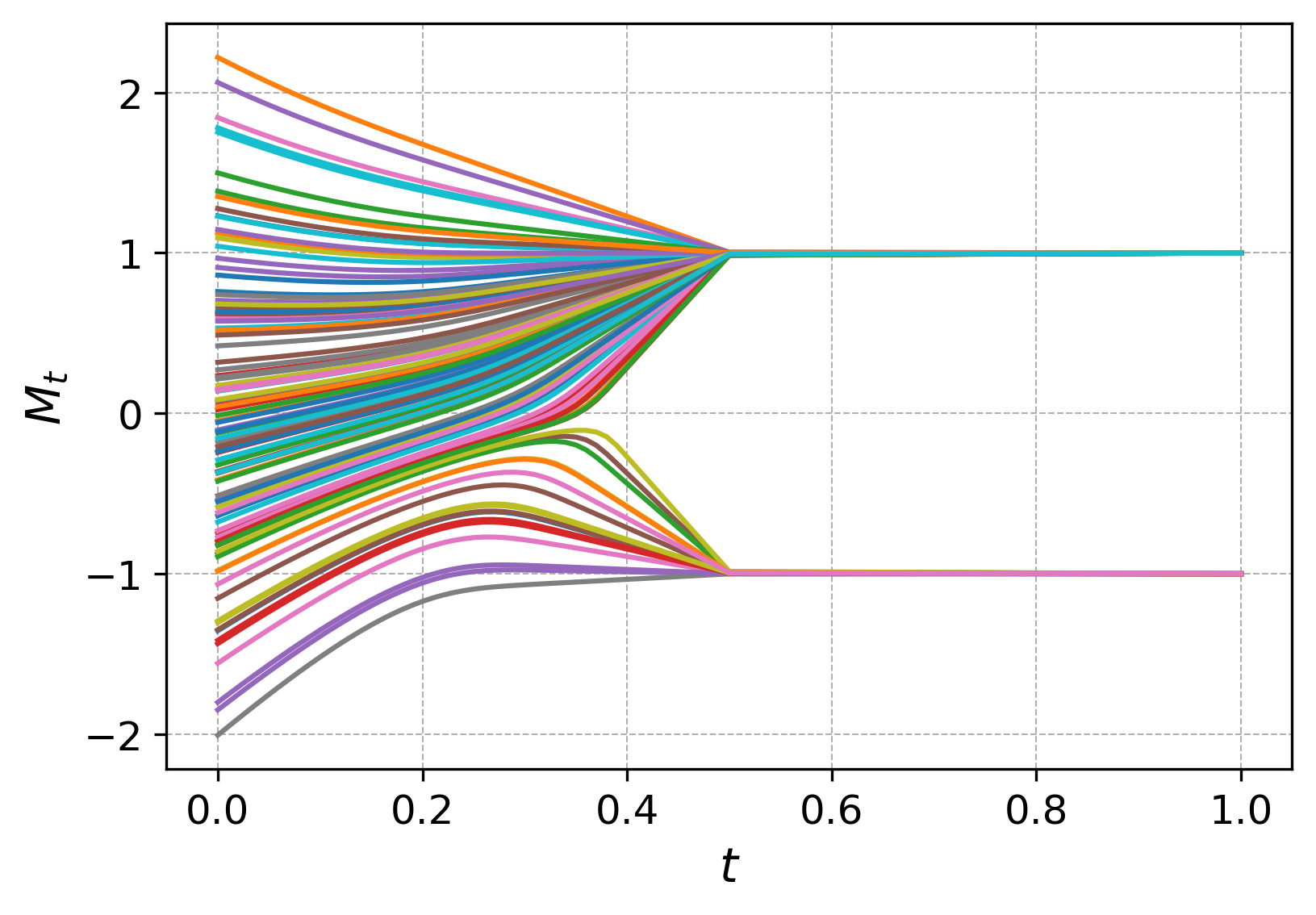}
    \includegraphics[width=.75\linewidth]{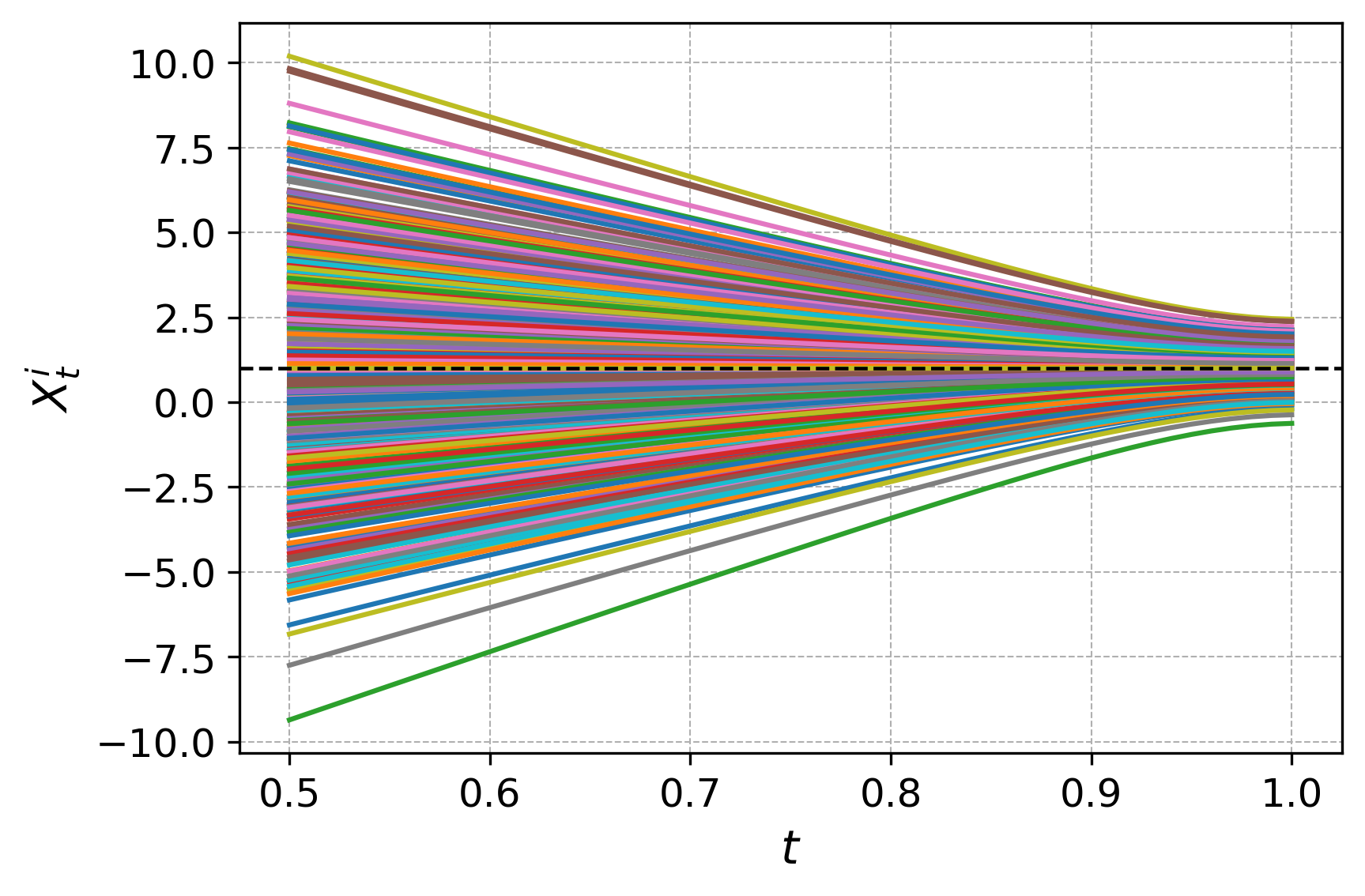}
    \includegraphics[width=.75\linewidth]{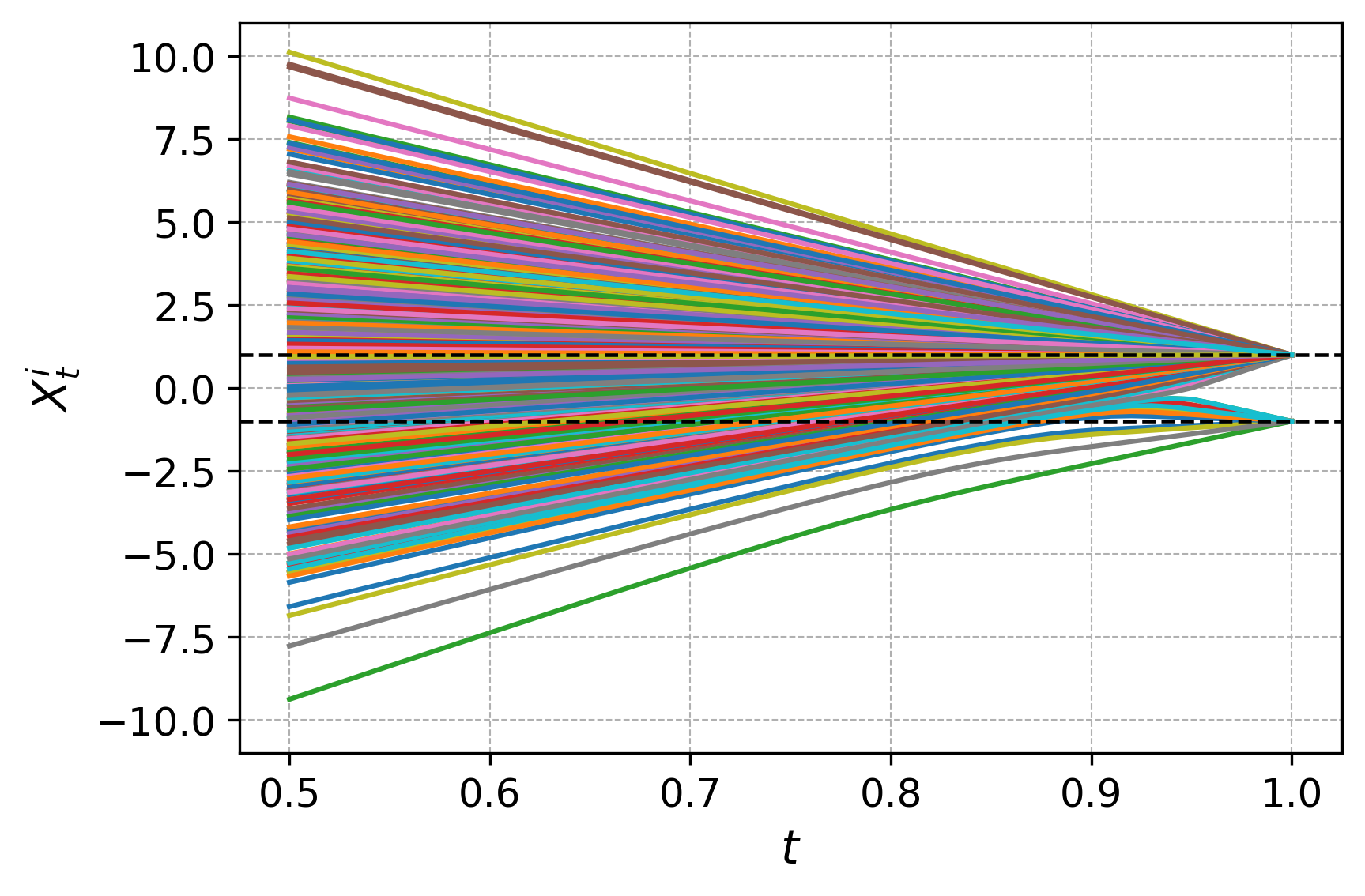}
    \caption{We run $100$ realizations $(X^{(j)}_t)_{j=1}^{100}$ of the probability flow ODE \eqref{eq:ode} associated with the dilated VE interpolant for the GM distribution, uniformly discretized with step size $d=10^6, \Delta t=0.01, \kappa=3,$ $\sigma^2=1/4,$ and $p=0.8.$ For each realization, we plot in the \textbf{top panel} $M^{(j)}_t=r\cdot X_t^{(j)}/d.$ We then take a single realization $X^{(1)}_t$ and plot in the \textbf{middle panel} the trajectory of the coordinates $(X^{(1)}_t)^i$ for $i=1, \cdots, 500$ in the second phase $t\in[1/2, 1].$ We do not plot $t\in[0,1/2]$ since $X^{(1)}_t$ is of order $\sqrt{d}$ there and would clutter the plot. For $t$ close to $1,$ the trajectories converge to a Gaussian centered at $1$ indicated by the dashed line. We also run the probability flow ODE associated with the dilated VE interpolant generating samples from the CW distribution with $d,\Delta t, \kappa$ as for the GM distribution and $\beta=2$. We then plot in the \textbf{bottom panel} the trajectories of the first $500$ coordinates of one realization, with dashed lines at $\pm 1.$}
    \label{fig:gm:dve}
\end{figure}

\subsection{VP and VE on real data: CelebA}
We showed earlier that, without time dilating, the high-level feature of the GM distribution given by the parameter $p$ is captured by the VE but not the VP interpolant, whereas the low-level feature given by $\sigma^2$ is captured by the VP but not the VE interpolant. The theoretical analysis we gave is for the GM distribution under the probability flow ODE. However, the following experiment shows that this behavior is still present when working with real image distributions using the VP and VE SDE from \cite{song2021scorebasedgenerativemodelingstochastic}.

We use the dataset CelebA-HQ from \cite{karras2018progressivegrowinggansimproved} consisting of $30,000$ images of faces from celebrities. We use models pretrained on this dataset for the VP and VE SDEs from \cite{song2021scorebasedgenerativemodelingstochastic}. Then, we generate images with the VP and VE SDEs and measure how well they reproduce the high- and low-level features of the dataset. 

For the low-level feature, we use a neural network from the Deepface library \cite{serengil2024lightface} that detects whether there is a face in the generated image. For the high-level feature, we use another neural network that predicts the race of the person, and calculate the KL divergence between the race distribution of the original dataset and the race distribution of a set of images generated by the VP or VE SDEs. (See the appendix for details.)

For a given number of discretization steps, we generate $7,500$ images with the VP and the VE SDEs using a uniform grid, and then measure the low- and high-level features as described before. The results are shown in Figure \ref{fig:celeba}. We observe that increasing the number of discretization steps for the VP SDE makes it better at capturing high-level features but does not improve the low-level features. Conversely, taking more discretization steps for the VE SDE improves how well it reproduces low-level features, but does not improve the high-level features. This is further evidence that the VP SDE solves the low-level feature ``by default'' whereas the VE SDE solves the high-level feature ``by default''.

\begin{figure}
    \centering
    \includegraphics[width=1\linewidth]{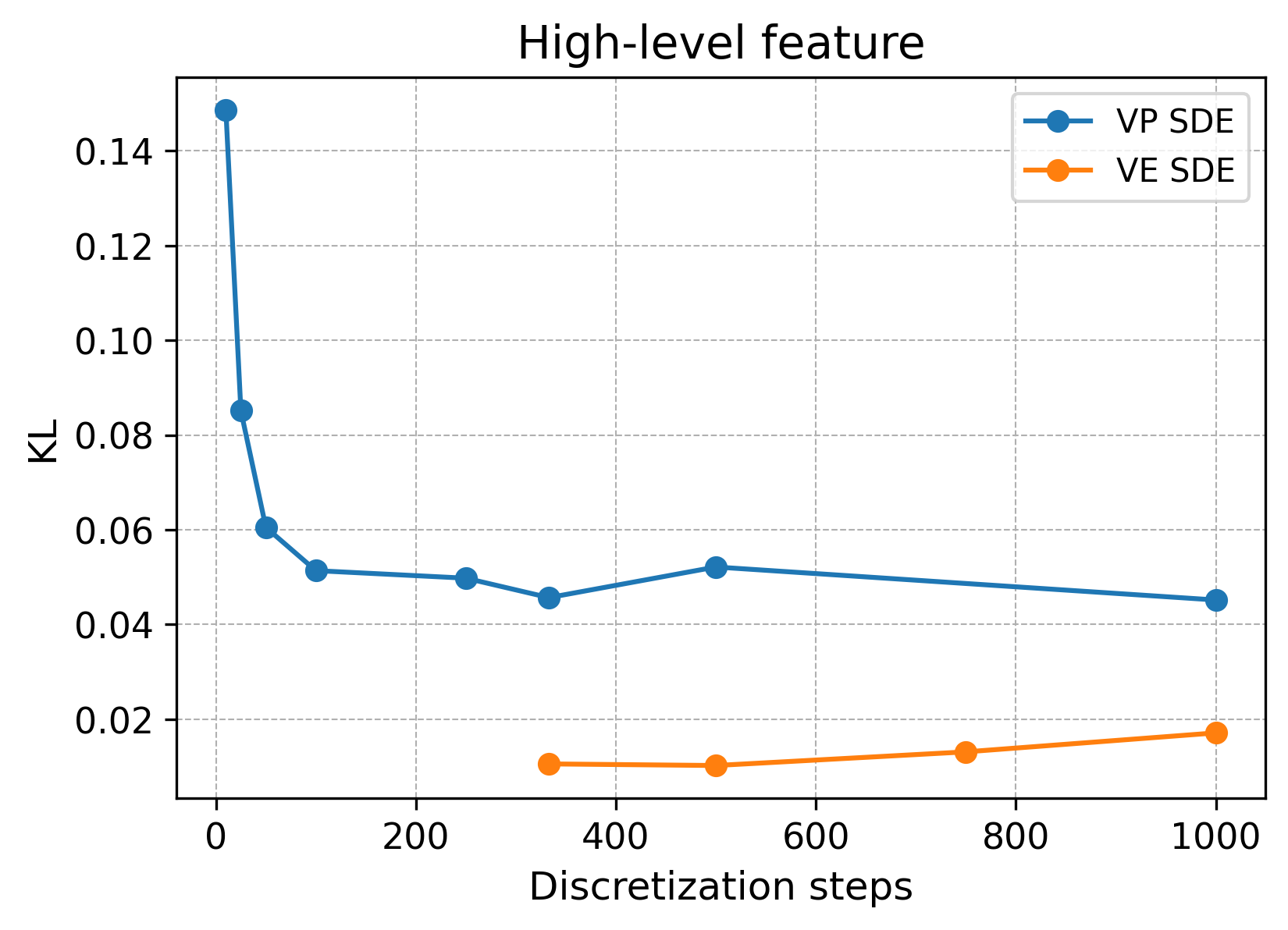}
    \includegraphics[width=1\linewidth]{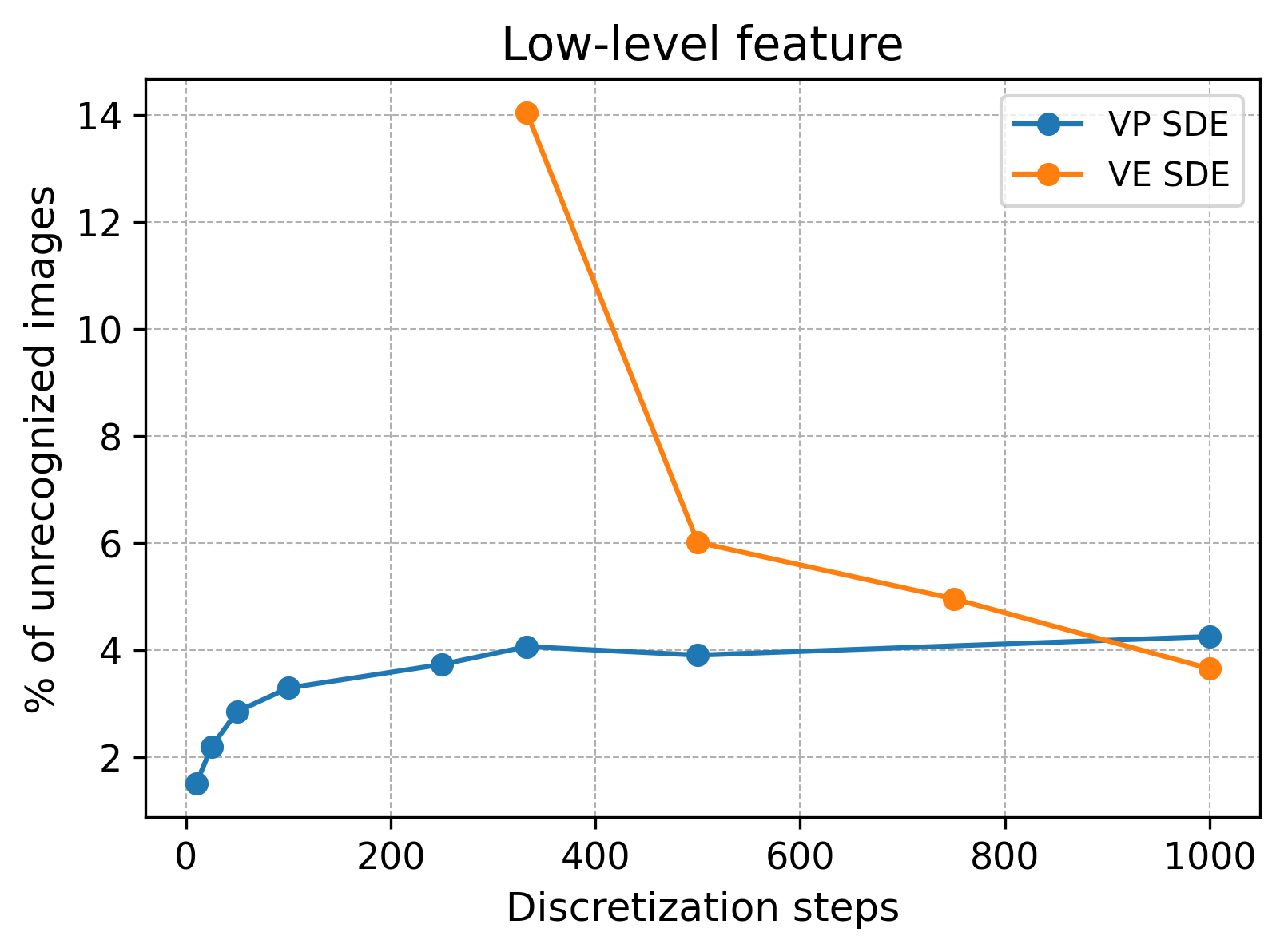}
    \caption{We plot, for different number of discretization steps and for the VP/VE SDEs, the KL divergence between the race distribution of a set of $7,500$ generated images and the race distribution of the original dataset in the \textbf{top panel} (high-level feature). In the \textbf{bottom panel} (low-level feature), we plot the percentage of the generated images that are classified as not containing a face.}
    \label{fig:celeba}
\end{figure}

\bibliographystyle{apalike}
\bibliography{biblio}

\onecolumn
\appendix

\section{General results}
We will need the following results that follow immediately from  \cite{albergo2023stochasticinterpolantsunifyingframework} (Appendix A)
\begin{lem}
\label{lem:gen:gen} Let $z\sim \mathcal{N}(0, \text{Id}_{d})$ and $a\sim p\mathcal{N}(r,\sigma^2\text{Id})+(1-p)\mathcal{N}(-r,\sigma^2\text{Id})$. Then the law of the interpolant $I_{\tau}=c\alpha_{\tau}z+\beta_{\tau}a$ coincides with the law of the solution of the probability flow ODE
\begin{equation}
\begin{aligned}\dot{X}_{\tau} & =\frac{c^2\alpha_{\tau}\dot{\alpha}_{\tau}+\sigma^2\beta_{\tau}\dot{\beta}_{\tau}}{c^2\alpha_{\tau}^{2}+\sigma^2\beta_{\tau}^{2}}X_{\tau} 
+ \frac{c^2\alpha_{\tau}(\alpha_{\tau}\dot{\beta}_{\tau}-\dot{\alpha}_{\tau}\beta_{\tau})}{c^2\alpha_{\tau}^{2}+\sigma^2\beta_{\tau}^{2}} r\tanh\left(h+\frac{\beta_{\tau} r\cdot X_{\tau}}{c^2\alpha_{\tau}^{2}+\sigma^2\beta_{\tau}^{2}}\right),\quad X_{0}\sim{\mathcal{N}}(0,c^2\text{Id}_{d})\end{aligned}
\label{eq:b:gmm:2m}
\end{equation}
where $h$ is such that $e^{h}/(e^{h}+e^{-h})=p.$
\end{lem} 

\begin{lem}
    \label{lem:g:to:g}
    Let $z\sim \mathcal{N}(0,1)$ and $a\sim \mathcal{N}(0,\sigma^2).$ Then the law of the interpolant $I_{\tau}=\alpha_{\tau}z+\beta_{\tau}a$ coincides with the law of the solution of the probability flow ODE
    \begin{equation}
        \dot{X}_{\tau} =\frac{\alpha_{\tau}\dot{\alpha}_{\tau}+\sigma^2\beta_{\tau}\dot{\beta}_{\tau}}{\alpha_{\tau}^{2}+\sigma^2\beta_{\tau}^{2}}X_{\tau} 
        ,\quad X_{0}\sim{\mathcal{N}}(0,1).
    \end{equation}
\end{lem}

\begin{lem}
    \label{lem:g:to:sgmm}
    Let $z\sim \mathcal{N}(0,1)$ and $a\sim p\mathcal{N}(m\kappa,1) + (1-p)\mathcal{N}(-m\kappa,1).$ Then the law of the interpolant $I_{\tau}=\alpha_{\tau}z+\beta_{\tau}a$ coincides with the law of the solution of the probability flow ODE
    \begin{equation}
        \dot{X}_{\tau} =\frac{\alpha_{\tau}\dot{\alpha}_{\tau}+\sigma^2\beta_{\tau}\dot{\beta}_{\tau}}{\alpha_{\tau}^{2}+\sigma^2\beta_{\tau}^{2}}X_{\tau} + \frac{\alpha_{\tau}(\alpha_{\tau}\dot{\beta}_{\tau}-\dot{\alpha}_{\tau}\beta_{\tau})}{\alpha_{\tau}^{2}+\sigma^2\beta_{\tau}^{2}} \kappa mr\tanh\left(mh+\frac{\beta_{\tau} \kappa mr\cdot X_{\tau}}{\alpha_{\tau}^{2}+\sigma^2\beta_{\tau}^{2}}\right)
        ,\quad X_{0}\sim{\mathcal{N}}(0,1).
    \end{equation}
    where $h$ is such that $e^{mh}/(e^{mh}+e^{-mh})=p.$
\end{lem}

\begin{lem}
    \label{lem:g:to:sgmm2}
    Let $z\sim \mathcal{N}(0,\kappa^2)$ and $a\sim \left(\tfrac{1+m}{2}\right)\delta_1 + \left(\tfrac{1-m}{2}\right)\delta_{-1}.$ Then the law of the interpolant $I_{\tau}=(1-\tau)z+a$ coincides with the law of the solution of the probability flow ODE
    \begin{equation}
        \dot{X}_{\tau} = \frac{-X_\tau + \tanh\left(\beta m + \frac{X_\tau}{\kappa^2(1-\tau)^2}\right)}{1-\tau}
        ,\quad X_{0}\sim{\mathcal{N}}(0,\kappa^2)
    \end{equation}
    where $m$ and $\beta$ fulfill $\tanh(\beta m) = m.$
\end{lem}

\section{Non-dilated interpolants fail at capturing either $p$ or $\sigma^2$ for the GM}
In this section we prove Proposition \ref{prp:vp} and \ref{prp:ve} from the main text, showing that without time-dilation, the VP and VE interpolant fail at either capturing $p$ or $\sigma^2.$
\begin{proof}[Proof of Proposition \ref{prp:vp}]
    \textit{(VP does not capture $p.$)} Consider the variance preserving interpolant $I_\tau = \alpha_\tau z+\beta_\tau a$ with $z\sim \mathcal{N}(0,\text{Id})$ and $a\sim \mu.$ Let $X_\tau$ be the solution of the probability flow ODE associated with $I_\tau$ given by Lemma \ref{lem:gen:gen}. If we let $M_\tau=r\cdot X_\tau/{d},$ we have for $\tau \in[0,1]$
    \begin{equation}
    \dot M_\tau= \frac{\alpha_{\tau}\dot{\alpha}_{\tau}+\sigma^2\beta_{\tau}\dot{\beta}_{\tau}}{\alpha_{\tau}^{2}+\sigma^2\beta_{\tau}^{2}}M_{\tau} 
    + \frac{\alpha_{\tau}(\alpha_{\tau}\dot{\beta}_{\tau}-\dot{\alpha}_{\tau}\beta_{\tau})}{\alpha_{\tau}^{2}+\sigma^2\beta_{\tau}^{2}} \tanh\left(h+\frac{\beta_{\tau}dM_{\tau}}{\alpha_{\tau}^{2}+\sigma^2\beta_{\tau}^{2}}\right)
    \label{eq:b:gmm:2m}
    \end{equation}
    with $M_{0}\sim{\mathcal{N}}(0,1/d).$ 
    
    For $\tau>0$, we get that for $d$ large, since $dM_\tau =\Omega(\sqrt{d})$
    \begin{align}
        \label{eq:mu:22}  
        \dot M_{\tau}= \frac{\alpha_{\tau}\dot{\alpha}_{\tau}+\sigma^2\beta_{\tau}\dot{\beta}_{\tau}}{\alpha_{\tau}^{2}+\sigma^2\beta_{\tau}^{2}}M_{\tau} 
        + \frac{\alpha_{\tau}(\alpha_{\tau}\dot{\beta}_{\tau}-\dot{\alpha}_{\tau}\beta_{\tau})}{\alpha_{\tau}^{2}+\sigma^2\beta_{\tau}^{2}} \sgn(M_{\tau}),
    \end{align}
    Fix $\tau_0=\Theta_d(1)$ positive. If we run this equation from $\tau=\tau_0$ til $\tau=1,$ the sign of $M_{\tau}$ will be preserved. This follows because since $\alpha_\tau(\alpha_{\tau}\dot{\beta}_{\tau}-\dot{\alpha}_{\tau}\beta_{\tau}) > 0$ for $\tau\in [0,1)$ means that whenever $M_{\tau}=o_d(1)$ the second term in the RHS of \eqref{eq:mu:22} will dominate implying that $\sgn(\dot M_{\tau}) = \sgn(M_{\tau}).$
    
    For $\tau=0,$ we have
    \begin{align}
        \dot{X}_{0} & = \dot \alpha_0 X_0 + \dot\beta_0 r\tanh(h).
    \end{align}
    If we integrate this ODE with a step of size $\epsilon = \Theta_d(1),$ we get 
    $$X_\epsilon = \epsilon (\dot \alpha_0 X_0 + \dot\beta_0 r\tanh(h))+X_0$$
    This means that $M_\epsilon = \dot\beta_0\epsilon  \tanh(h)+O_d(1/\sqrt{d}).$ Since after this step the sign of $M_{\tau}$ will be preserved, we have that as $d\to\infty$
    \begin{itemize}
        \item $p>1/2\implies h>0\implies$ all samples will go to the $+r$ mode.
        \item $p=1/2\implies h=0\implies$ half of the samples will go to $+r$ and half to $-r.$
        \item $p<1/2\implies h<0\implies$ all samples will go to the $-r$ mode.
    \end{itemize}
    
    \textit{(VP captures $\sigma^2$)} Let $X^\perp_\tau = X^{\Delta t}_\tau - M^{\Delta t, d}_\tau r.$ We note the problem is symmetric in the orthogonal complement of $r,$ so we expect $X^\perp_\tau$ to be the right object to look at. We have again from Lemma \ref{lem:gen:gen} that 
    \begin{align}
        \label{eq:x_2}
        \dot X^\perp_\tau= \frac{\alpha_{\tau}\dot{\alpha}_{\tau}+\sigma^2\beta_{\tau}\dot{\beta}_{\tau}}{\alpha_{\tau}^{2}+\sigma^2\beta_{\tau}^{2}}X^\perp_{\tau},\quad X^\perp_0\sim \mathcal{N}(0,\text{Id}_{d-1})
    \end{align}
    We note that for any $d$ and discretization step size $\Delta \tau,$ we have that $X^\tau_{k\Delta \tau}$ will remain Gaussian. We then have 
    \begin{align*}
        X^{\perp}_\tau \sim \mathcal{N}(0,\left( \sigma^{\Delta \tau}_\tau\right)^2\text{Id}_{d-1}).
    \end{align*}
    To determine $\sigma^{\Delta \tau}_\tau,$ we look at one coordinate $i\in \{1,\cdots, d-1\}$ and note that the resulting velocity field is independent of $d$. Also, Lemma \ref{lem:g:to:g} shows that this is the $1$-dimensional velocity field corresponding to the interpolant $I_\tau=\alpha_\tau z + \beta_\tau a$ that transports $z\sim\mathcal{N}(0, 1)$ to $a\sim {\mathcal{N}}(0, \sigma^2)$ so we are done.
\end{proof}

\begin{proof}[Proof of Proposition \ref{prp:ve}]
    \textit{(VE captures $p.$)}
    Let $I_\tau$ be the variance exploding interpolant $I_\tau = \sqrt{d}\alpha_\tau z + \beta_\tau a$ with $z\sim \mathcal{N}(0,\text{Id})$ and $a\sim \mu.$ Let $X_\tau$ be the solution of the probability flow ODE associated with $I_\tau$ given by Lemma \ref{lem:gen:gen}. 
    For $\tau \in [0,1],$ we have
    \begin{equation}
        \dot{M}_{\tau} =\frac{d\alpha_{\tau}\dot{\alpha}_\tau+\sigma^2\beta_\tau\dot{\beta}_\tau}{d\alpha_\tau^{2}+\sigma^2\beta_\tau^{2}}M_\tau 
        + \frac{d\alpha_\tau(\alpha_\tau\dot{\beta}_\tau-\dot{\alpha}_\tau\beta_\tau)}{d\alpha_\tau^{2}+\sigma^2\beta_\tau^{2}}  r\tanh\left(h+\frac{\beta_\tau dM_\tau}{d\alpha_\tau^{2}+\sigma^2\beta_\tau^{2}}\right),\quad M_{0}\sim{\mathcal{N}}(0,1).
        \label{eq:mag:ve}
    \end{equation}
    Taking the $d\to\infty$ limit of this equation gives for $t\in[0,1)$
    \begin{equation}
        \dot{M}_\tau = \frac{\dot\alpha_\tau}{\alpha_\tau}M_\tau+ \frac{\alpha_\tau\dot\beta_\tau-\dot\alpha_\tau\beta_\tau}{\alpha_\tau}\tanh\left(h+\frac{\beta_\tau M_\tau}{\alpha^2_\tau}\right).
        \label{eq:final:m_t2}
    \end{equation}
    Hence we get a well-defined equation for the magnetization. Define the speciation time $\tau_s$ as the time in the generative process after which the mode of the samples is determined. Since $M_t$ determines the mode of the samples, and we obtained a $d$-independent limiting equation for $M_t,$ we get that $\tau_s \in (0,1)$ as $d$ to $\infty.$ Moreover, from Lemma \ref{lem:gen:gen} we know that the ODE for $M_t$ from equation \eqref{eq:final:m_t2} corresponds to the 1-dimensional velocity field that transports ${\mathcal{N}}(0,1)$ to $p{\mathcal{N}}(1,0) + (1-p){\mathcal{N}}(-1, 0).$
    
\textit{(VE does not capture $\sigma^2$.)}
    We have 
    \begin{equation}
        \dot{X}^\perp_\tau =\frac{d\alpha_\tau\dot{\alpha}_\tau+\sigma^2\beta_\tau\dot{\beta}_\tau}{d\alpha_\tau^{2}+\sigma^2\beta_\tau^{2}}X^\perp_\tau,\quad X_{0}\sim{\mathcal{N}}(0,d\text{Id}_{d-1}).
        \label{eq:mag:ve}
    \end{equation}
    As in the proof of Proposition 1, we see that $X^\perp_\tau$ is Gaussian for any $d$ and $\Delta t.$ Let $w_1,\cdots,w_{d-1}$ be an orthonormal basis of the complement of $r.$ Then for $i\in \{1,\cdots, d-1\}$ with $\nu^i_t = w_i\cdot X_t / \sqrt{d}$ we have
    \begin{equation}
        \dot \nu^i_t = \frac{d\alpha_\tau\dot{\alpha}_\tau+\sigma^2\beta_\tau\dot{\beta}_\tau}{d\alpha_\tau^{2}+\sigma^2\beta_\tau^{2}}\nu_t^i,\quad \nu^i_t \sim{\mathcal{N}}(0,1).
    \end{equation}
    In the $d$ limit, we get 
    \begin{equation}
        \dot \nu^i_t = \frac{\dot{\alpha}_\tau}{\alpha_\tau}\nu_t^i,\quad \nu^i_t \sim{\mathcal{N}}(0,1).
    \end{equation}
    From Lemma \ref{lem:g:to:g} we know that this ODE transports $\mathcal{N}(0,1)$ to $\delta_0.$\\
\end{proof}

\section{Speciation time for the VP interpolant}
In this section, we state and prove Proposition \ref{prp:spec:vp}, showing that without time dilating, the VP interpolant has a speciation time that goes to zero as $d$ goes to infinity.

\begin{prop}
    \label{prp:spec:vp}
    Consider the variance preserving interpolant $I_\tau=\alpha_\tau z + \beta_\tau a$ where $z\sim \mathcal{N}(0,\text{Id})$ and $a \sim p\mathcal{N}(r,\sigma^2\text{Id})+(1-p)\mathcal{N}(-r,\sigma^2\text{Id})$ and let $X_\tau$ be its associated probability flow ODE from equation (1) in the main text. Then, the speciation time $\tau_s$ (i.e. the time in the generative process after which the mode of the samples is determined) goes to zero as $d$ goes to infinity. In particular, the interpolant $(1-\tau)z+\tau a$ has speciation time $1/\sqrt{d}.$
\end{prop}
\begin{proof}
    We have from Lemma \ref{lem:gen:gen} that $X_\tau$ fulfills the ODE
    \begin{equation}
        \begin{aligned}\dot{X}_{\tau} & =\frac{\alpha_{\tau}\dot{\alpha}_{\tau}+\sigma^2\beta_{\tau}\dot{\beta}_{\tau}}{\alpha_{\tau}^{2}+\sigma^2\beta_{\tau}^{2}}X_{\tau} 
        + \frac{\alpha_{\tau}(\alpha_{\tau}\dot{\beta}_{\tau}-\dot{\alpha}_{\tau}\beta_{\tau})}{\alpha_{\tau}^{2}+\sigma^2\beta_{\tau}^{2}} r\tanh\left(h+\frac{\beta_{\tau} r\cdot X_{\tau}}{\alpha_{\tau}^{2}+\sigma^2\beta_{\tau}^{2}}\right),\quad X_{0}\sim{\mathcal{N}}(0,\text{Id}_{d}).
        \label{eq:x:ode:prp}
        \end{aligned}
    \end{equation}
    Let $\mu_\tau=r\cdot X_\tau/\sqrt{d}.$ We then have      
    \begin{equation}
    \mu_{\tau}\stackrel{d}{=}\alpha_\tau Z+\sqrt{d}\beta_\tau m\label{eq:mag:van:1}
    \end{equation}
    where $Z\sim \mathcal{N}(0,1)$ and $m:=r\cdot a/d =\Theta_{d}(1).$ Let us calculate $\tau_0,$ the time where the terms $\alpha_\tau Z$ and $\sqrt{d}\beta_\tau m$ are of the same order. Since $Z$ and $m$ are $\Theta_d(1),$ we are interested in finding $\tau$ such that $\alpha_\tau \approx \sqrt{d}\beta_\tau$, or equivalently $\frac{1}{\sqrt{d}} \approx \frac{\beta_\tau}{\alpha_\tau}.$ Since ${\beta_0/\alpha_0}=0,$ $(\beta_\tau/\alpha_\tau)'>0,$ and $\alpha_\tau,\beta_\tau$ are independent of $d,$ we get 
    \begin{align}
        \tau_0=o_d(1).
    \end{align}
    We have for $\tau \ll \tau_0$ 
    \begin{equation}
    \sqrt{d}\beta_\tau m\ll\alpha_\tau Z,\label{eq:hat_m_t:dominates:tm}
    \end{equation}
    and for $\tau \gg\tau_0,$ 
    \begin{equation}
    \sqrt{d}\beta_\tau m\gg\alpha_\tau Z.\label{eq:tm:dominates:hat_m_t}
    \end{equation}
    This means there is a transition in what term dominates in $\mu_{\tau}$ at $\tau=\tau_0.$ This will imply, as formalized below, that each sample will \textit{speciate} to one of the two modes for $\tau \approx\tau_0,$ and it will remain in that mode for $\tau\gg\tau_0.$

    From equation \eqref{eq:x:ode:prp} we get
    \begin{equation}
        \begin{aligned}\dot{\mu}_{\tau} & =\frac{\alpha_{\tau}\dot{\alpha}_{\tau}+\sigma^2\beta_{\tau}\dot{\beta}_{\tau}}{\alpha_{\tau}^{2}+\sigma^2\beta_{\tau}^{2}}\mu_{\tau} 
        + \frac{\alpha_{\tau}(\alpha_{\tau}\dot{\beta}_{\tau}-\dot{\alpha}_{\tau}\beta_{\tau})}{\alpha_{\tau}^{2}+\sigma^2\beta_{\tau}^{2}} \sqrt{d}\tanh\left(h+\frac{\beta_{\tau} \sqrt{d}\mu_\tau}{\alpha_{\tau}^{2}+\sigma^2\beta_{\tau}^{2}}\right),\quad \mu_{0}\sim{\mathcal{N}}(0,1).\end{aligned}
    \end{equation}
    We write $\mu_\tau$ in terms of a potential $\dot \mu_\tau=-\partial_\mu V_\tau(\mu_\tau)$ with $V_\tau$
    \begin{align}
        V_\tau(\mu) = -\frac{\alpha_{\tau}\dot{\alpha}_{\tau}+\sigma^2\beta_{\tau}\dot{\beta}_{\tau}}{2(\alpha_{\tau}^{2}+\sigma^2\beta_{\tau}^{2})} \mu^2 -\frac{\alpha_{\tau}(\alpha_{\tau}\dot{\beta}_{\tau}-\dot{\alpha}_{\tau}\beta_{\tau})}{\beta_\tau} \log\cosh\left(h+\frac{\beta_{\tau} \sqrt{d}\mu}{\alpha_{\tau}^{2}+\sigma^2\beta_{\tau}^{2}}\right)
        \label{eq:potential}
    \end{align}
    
    Take $\tau\ll\tau_0.$ We have by equation \eqref{eq:hat_m_t:dominates:tm} that $\frac{\beta_{\tau} \sqrt{d}\mu_\tau}{\alpha_{\tau}^{2}+\sigma^2\beta_{\tau}^{2}} \approx \frac{\beta_{\tau} \sqrt{d}\alpha_\tau Z}{\alpha_{\tau}^{2}+\sigma^2\beta_{\tau}^{2}} \ll \frac{\alpha^2_\tau Z}{\alpha_{\tau}^{2}+\sigma^2\beta_{\tau}^{2}} = O_d(1).$ This means we can Taylor expand the $\log\cosh$ term in ${V_\tau}$ to get 
    \begin{align}
        V_\tau(\mu) = -\frac{\alpha_{\tau}\dot{\alpha}_{\tau}+\sigma^2\beta_{\tau}\dot{\beta}_{\tau}}{2(\alpha_{\tau}^{2}+\sigma^2\beta_{\tau}^{2})} \mu^2 -\frac{\alpha_{\tau}(\alpha_{\tau}\dot{\beta}_{\tau}-\dot{\alpha}_{\tau}\beta_{\tau})}{\alpha_{\tau}^{2}+\sigma^2\beta_{\tau}^{2}}\sqrt{d} \tanh\left(h\right)\mu + C_\tau.
    \end{align}
    This is a quadratic well shifted away from the origin which will generate the asymmetry in the relative weights of the modes.

    Take $\tau\gg\tau_0.$ We have by equation \eqref{eq:tm:dominates:hat_m_t} that $\frac{\beta_{\tau} \sqrt{d}\mu_\tau}{\alpha_{\tau}^{2}+\sigma^2\beta_{\tau}^{2}} \approx\frac{\beta_{\tau}^2 dm}{\alpha_{\tau}^{2}+\sigma^2\beta_{\tau}^{2}} \gg 1$ since either $\beta_\tau > \alpha_\tau$ and then $\frac{\beta_{\tau}^2 dm}{\alpha_{\tau}^{2}+\sigma^2\beta_{\tau}^{2}} = \Theta(d)$ or $\beta_\tau \leq \alpha_\tau$ and then $\tau\gg\tau_0$ implies $\beta_\tau \sqrt{d} \gg \alpha_\tau$ meaning that $\frac{\beta_{\tau}^2 dm}{\alpha_{\tau}^{2}+\sigma^2\beta_{\tau}^{2}} \gg \frac{\alpha^2_\tau m}{\alpha_{\tau}^{2}+\sigma^2\beta_{\tau}^{2}} = \Theta_d(1).$ This means we can approximate \eqref{eq:potential} as follows
    \begin{align}
        V_\tau(\mu) = -\frac{\alpha_{\tau}\dot{\alpha}_{\tau}+\sigma^2\beta_{\tau}\dot{\beta}_{\tau}}{2(\alpha_{\tau}^{2}+\sigma^2\beta_{\tau}^{2})} \mu^2 -\frac{\alpha_{\tau}(\alpha_{\tau}\dot{\beta}_{\tau}-\dot{\alpha}_{\tau}\beta_{\tau})\sqrt{d}}{\alpha_{\tau}^{2}+\sigma^2\beta_{\tau}^{2}} |\mu|.
    \end{align}
    This is a symmetric double well structure. Under this potential, the mode that each sample will belong to is determined since the beginning. The relative asymmetry of the modes given by $h$ (which is a function of $p$) does not appear in this potential anymore.
    
    We conclude that the speciation time $\tau_s=\tau_0=o_d(1).$ In particular, if $\alpha_\tau=1-\tau$ and $\beta_\tau = \tau,$ we get that $\tau_s=\tau_0=1/\sqrt{d}.$
\end{proof}

\begin{prop}
    Let $X^{\Delta \tau}_\tau$ be obtained from the probability flow ODE from equation (1) in the main text associated with the VP interpolant $I_\tau=(1-\tau)z+\tau a$ where  $z\sim \mathcal{N}(0,\text{Id})$ and $a \sim p\mathcal{N}(r,\sigma^2\text{Id})+(1-p)\mathcal{N}(-r,\sigma^2\text{Id}).$ Consider running this ODE with a uniform grid with step size $\Delta \tau(d) = o(\sqrt{d}).$ Let $M_\tau^{\Delta \tau, d}=r\cdot X_\tau^{\Delta \tau}/d.$ Then
    \begin{align*}
        \lim_{d\to\infty} M^{\Delta \tau, d}_1 \stackrel{d}{=} \hat p \delta_1 + (1-\hat p) \delta_{-1}
    \end{align*}
    where
    \begin{align*}
        \hat p =
        \begin{cases}
            1 \qquad \text{ if } p>1/2\\
            1/2 \quad \text{ if } p=1/2\\
            0 \qquad \text{ if } p<1/2\
        \end{cases}
    \end{align*}
    \label{prp:spec2}
\end{prop}
\begin{proof}
    We proceed similarly to the proof of Proposition 1. Let $M_\tau=r\cdot X_\tau/{d},$ then we have for $\tau \in[0,1]$
    \begin{equation}
    \dot M_\tau= \frac{(1+\sigma^2)\tau-1}{(1-\tau)^{2}+\sigma^2{\tau}^{2}}M_{\tau} 
    + \frac{1-\tau}{(1-\tau)^{2}+\sigma^2{\tau}^{2}} \tanh\left(h+\frac{{\tau}dM_{\tau}}{(1-\tau)^{2}+\sigma^2{\tau}^{2}}\right)
    \label{eq:b:gmm:2m}
    \end{equation}
    with $M_{0}\sim{\mathcal{N}}(0,1/d).$ 
    
    For $\tau=k\Delta \tau$ with $k$ positive integer, we get that $\tau dM_\tau$ goes to infinity as $d$ grows. This means that for $d$ large
    \begin{align}
        \label{eq:mu:2}  
        \dot M_\tau= \frac{(1+\sigma^2)\tau-1}{(1-\tau)^{2}+\sigma^2{\tau}^{2}}M_{\tau} 
    + \frac{1-\tau}{(1-\tau)^{2}+\sigma^2{\tau}^{2}} \sgn(M_\tau).
    \end{align}
    Fix $\tau_0=k\Delta \tau$ with $k$ positive integer again. If we run this equation from $\tau=\tau_0$ til $\tau=1,$ the sign of $M_{\tau}$ will be preserved. This follows because whenever $M_{\tau}=o_d(1)$ the second term in the RHS of \eqref{eq:mu:2} will dominate implying that $\sgn(\dot M_{\tau}) = \sgn(M_{\tau}).$
    
    For $\tau=0,$ we have
    \begin{align}
        \dot{X}_{0} & = -X_0 + r\tanh(h).
    \end{align}
    If we integrate this ODE with step size $\Delta \tau$ we get 
    $$X_{\Delta \tau} = \Delta \tau (- X_0 + r\tanh(h))+X_0$$
    This means that $M_{\Delta \tau} = \Delta \tau \tanh(h)+O_d(1/\sqrt{d}).$ Since after this step the sign of $M_{\tau}$ will be preserved, we have that as $d\to\infty$
    \begin{itemize}
        \item $p>1/2\implies h>0\implies$ all samples will go to the $+r$ mode.
        \item $p=1/2\implies h=0\implies$ half of the samples will go to $+r$ and half to $-r.$
        \item $p<1/2\implies h<0\implies$ all samples will go to the $-r$ mode.
    \end{itemize}
\end{proof}

\section{Dilated interpolants capture $p$ and $\sigma^2$ for the GM}
In this section, we prove Theorems \ref{thm:dvpa} and \ref{prp:dve} from the main text which show that the dilated VP and VE interpolant can recover $p$ and $\sigma^2$ when time-dilated.

\begin{proof}[Proof of Theorem \ref{thm:dvpa}]
    Consider the dilated variance preserving interpolant $I^P_t = (1-\tau_t)z + \tau_t a$ where $z\sim \mathcal{N}(0,\text{Id}),$ $a\sim \mu,$ and $\tau_t$ is given in equation (4) in the main text. Plugging in $\alpha_t=1-\tau_t$ and $\beta_t=\tau_t$ into the velocity field given by Lemma \ref{lem:gen:gen} yields 
    \begin{align}
    \label{eq:gen:dil3}
    \dot X_t= \frac{-(1-\tau_t)\dot\tau_t+\sigma^2\tau_t\dot\tau_t}
    {(1-\tau_t)^{2}+\sigma^2\tau_t^2}X_{t} 
    + \frac{(1-\tau_t)\dot\tau_t}{(1-\tau_t)^{2}+\sigma^2\tau_t^2} r\tanh\left(h+\frac{\tau_{t} r\cdot X_{t}}{(1-\tau_t)^{2}+\sigma^2\tau_t^2}\right)    
    \end{align}
    \textbf{First phase.} For $t\in [0,1/2],$ we have $\tau_t = \frac{2 \kappa t}{\sqrt{d}}.$ Plugging in into equation \eqref{eq:gen:dil3} gives
    \begin{align}
        \dot{X}_{t} & = -\frac{2\kappa}{\sqrt{d}}X_t +\frac{2\kappa}{\sqrt{d}} r\tanh\left(h+2\kappa t \frac{r\cdot X_{t}}{\sqrt{d}}\right) + O\left(\frac{1}{d}\right).
    \end{align}
    We then have with $\mu_t=r\cdot X_t/\sqrt{d},$
    \begin{align}
        \label{eq:vp:1st:mu}
        \dot{\mu}_{t} & = 2\kappa\tanh\left(h+2\kappa t \mu_t\right) + O\left(\frac{1}{\sqrt{d}}\right).
    \end{align}
    Taking $d\to\infty$ yields the limiting ODE for the $\mu_t$ in the first phase. By reparameterizing time $t(s)=s/2$ with $t:[0,1]\to[0,1/2],$ we get from Lemma \ref{lem:g:to:sgmm} (with $m=1$) that this the $1$-dimensional velocity field associated to the interpolant $I_s=\sqrt{1-s^2}z+sa$ that transports $z\sim{\mathcal{N}}(0,1)$ at $t(s=0)=0$ to $a\sim p{\mathcal{N}}(\kappa,1) + (1-p){\mathcal{N}}(-\kappa, 1)$ at $t(s=1)=1/2.$

    Let $X^\perp_t = X^{\Delta t}_t - \frac{r\cdot X^{\Delta t}_t}{{d}}r.$ We have from equation \eqref{eq:gen:dil3}
    \begin{align}
        \label{eq:x_perp}
        \dot{X}^\perp_{t} & = \frac{-(1-\tau_t)\dot\tau_t+\sigma^2\tau_t\dot\tau_t}
    {(1-\tau_t)^{2}+\sigma^2\tau_t^2} X^\perp_t.
    \end{align}
    Since this is a linear ODE with initial condition Gaussian, we have 
    \begin{align}
        \dot{X}^\perp_{t} \sim \mathcal{N}\left(0, \left(\hat \sigma^{\Delta t, d}_t\right)^2\text{Id}_{d-1}\right).
    \end{align}
    for any $d$ and $\Delta t.$ Further, using equation \eqref{eq:x_perp} with $\tau_t=\frac{2\kappa t}{\sqrt{d}}$ gives $\dot{X}^\perp_{t} = O({1/\sqrt{d}})$ meaning that for $t\in[0,1/2]$
    \begin{align}
        \lim_{\Delta t\to 0}\lim_{d \to\infty} \sigma^{\Delta t, d}_t =1. 
    \end{align}
    \newline
    \textbf{Second phase.} For $t\in[1/2,1],$ we have $\tau_t = \left(1-\frac{\kappa}{\sqrt{d}}\right)(2t-1)+\frac{\kappa}{\sqrt{d}}.$ Using equation \eqref{eq:gen:dil} again gives
    \begin{align}
        \label{eq:gen:dil:2nd}
        \dot X_t= \frac{-(1-t)+\sigma^2(t-\tfrac{1}{2})}
        {(1-t)^2+\sigma^2(t-\tfrac{1}{2})^2}X_{t} 
        + \frac{(1-t) r\tanh\left(h+\frac{(2t-1) r\cdot X_{t}+\kappa \frac{r\cdot X_t}{\sqrt{d}}}{(2-2t)^2+\sigma^2(2t-1)^2}\right)  }{(1-t)^2+\sigma^2(t-\tfrac{1}{2})^2} + O\left(\frac{1}{d}\right).
    \end{align}
    
    Writing $\mu_t=\frac{r\cdot X_t}{\sqrt{d}},$ this implies
    \begin{align}
        \dot \mu_t= \frac{-(1-t)+\sigma^2(t-\tfrac{1}{2})}
        {(1-t)^2+\sigma^2(t-\tfrac{1}{2})^2}\mu_{t} 
        + \frac{\sqrt{d}(1-t)\tanh\left(h+\frac{(2t-1) \sqrt{d}\mu_t +\kappa  \mu_t}{(2-2t)^2+\sigma^2(2t-1)^2}\right)}{(1-t)^2+\sigma^2(t-\tfrac{1}{2})^2} + O\left(\frac{1}{\sqrt{d}}\right).
        \label{eq:ode:mu}
    \end{align}
    From equation \ref{eq:vp:1st:mu} of the first phase, we have that for finite $d$ and discretizing with a step size of $\Delta t,$ we get 
    \begin{align}
        \mu_{t=1/2} = \theta + O\left(\frac{1}{\sqrt{d}}\right) + o_{\Delta t}(1)
    \end{align}
    where $\theta \sim p\mathcal{N}(\kappa, 1) + (1-p)\mathcal{N}(-\kappa, 1)$ and the term $o_{\Delta t}(1)$ goes to zero as $\Delta t$ goes to zero independently of $d,$ since this error only comes from discretizing the $d$-independent ODE $\dot{\mu}_{t} = 2\kappa\tanh\left(h+2\kappa t \mu_t\right)$ with $\mu_{t=0}\sim \mathcal{N}(0,1).$

    At $t=1/2,$ the argument of the $\tanh$ is $h+\kappa\mu_{1/2}.$ Assume $\theta$ takes value on the $+\kappa$ mode. For $d$ large enough and $\Delta t$ small enough (independently of $d$) we have that $|\mu_{1/2} - \theta| < 1$. We also have that $h\ll \kappa \theta$ and hence $h\ll \kappa\mu_t,$ where both inequalities hold with probability going to $1$ as $\kappa$ goes to infinity. This means we can approximate the ODE for $\mu_t$ for $t=1/2$ as 
    \begin{align}
        \dot \mu_t= \frac{-(1-t)+\sigma^2(t-\tfrac{1}{2})}
        {(1-t)^2+\sigma^2(t-\tfrac{1}{2})^2}\mu_{t} 
        + \frac{\sqrt{d}(1-t)\sgn(\mu_t)}{(1-t)^2+\sigma^2(t-\tfrac{1}{2})^2} + O\left(\frac{1}{\sqrt{d}}\right).
        \label{eq:ode:mu2}
    \end{align}
    We note that this remains valid for $t>1/2$ since under the approximation we used in equation \eqref{eq:ode:mu2}, we have that $\mu_t$ is increasing. Indeed, whenever $\mu_t=o(\sqrt{d}),$ the second term in the RHS of \eqref{eq:ode:mu2} will dominate. If $b$ takes value on the $-\kappa$ mode instead, an analogous argument shows that \eqref{eq:ode:mu2} is also valid in that case.
    
    We then use this approximation in the ODEs for $X_t$ to get for $t\in(1/2, 1)$
    \begin{align}
        \label{eq:final:x_t}
        \dot X_t= \frac{-(1-t)+\sigma^2(t-\tfrac{1}{2})}
        {(1-t)^2+\sigma^2(t-\tfrac{1}{2})^2}X_{t} 
        + \frac{(1-t) r\sgn(M_t)}{(1-t)^2+\sigma^2(t-\tfrac{1}{2})^2} + O\left(\frac{1}{d}\right),
    \end{align}
    where $M_t=r\cdot X_t/d.$ This yields the limiting ODE for $M_t$ in the theorem statement. We recall that from the analysis of the first phase (after taking the limit first on $d\to\infty$ and then on $\Delta t\to 0$) we got 
    \begin{align}
        \mu_{t=1/2}\sim p\mathcal{N}(\kappa, 1) + (1-p)\mathcal{N}(-\kappa, 1).
    \end{align}
    We argued above that the sign of $\mu_t$ will be preserved for $t\in[1/2, 1]$ with probability going to $1$ as $\kappa$ tends to $\infty.$ This means that 
    \begin{align}
        M_1 = p^\kappa\delta_1 + (1-p^\kappa)\delta_{-1}
    \end{align}
    where $p^\kappa$ is such that $\lim_{\kappa\to \infty}p^\kappa=p.$ 
    
    Let $X^\perp = X^{\Delta t}_t - \frac{r\cdot X^{\Delta t}_t}{{d}}r$ and note that
    \begin{align}
        \label{eq:final:x_t2}
        \dot X^{\perp}_t= \frac{-(1-t)+\sigma^2(t-\tfrac{1}{2})}
        {(1-t)^2+\sigma^2(t-\tfrac{1}{2})^2}X^{\perp}_{t}.
    \end{align}
    Since this is a linear ODE from a Gaussian initial condition, we have
    \begin{align}
    X^\perp_t \sim \mathcal{N}\left(0, \left(\hat \sigma^{\Delta t, d}_t\right)^2\text{Id}_{d-1}\right).
    \end{align}
    Under the change of variables $t(s)=s/2+1/2,$ the equation \eqref{eq:final:x_t2} for $X_t^\perp$ becomes
    \begin{align}
        \label{eq:final:x_t3}
        \dot X^{\perp}_s= \frac{-(1-s)+\sigma^2s}
        {(1-s)^2+\sigma^2s^2}X^{\perp}_{s}.
    \end{align}
    By taking one coordinate $i\in\{1,\cdots, d-1\}$ of $X_s^\perp$ we get from Lemma \ref{lem:g:to:g} that this is the velocity field associated with the interpolant $I_s=(1-s)z+sa$ where $z\sim \mathcal{N}(0,1)$ is transported to $a\sim \mathcal{N}(0, \sigma^2).$ For fixed $s\in [0, 1],$ the interpolant $I_s$ has variance $(1-s)^2+\sigma^2s^2 = (2-2t)^2+\sigma^2(2t-1)^2$ as claimed.
\end{proof}

We now turn to the proof of Theorem \ref{prp:dve} proving that the dilated variance exploding interpolant yields correct estimation of $p$ and $\sigma^2.$

\begin{proof}[Proof of Theorem \ref{prp:dve}] 
    Consider the variance exploding interpolant $I^E_t=\sqrt{d}\sqrt{1-\tau_t}z+\tau_t a$ with the time dilation given by equation \eqref{eq:t_dil:ve} in the main text. Plugging in $\alpha_t=1-\tau_t,$ $\beta_t=\tau_t,$ and $c=\sqrt{d}$ into the velocity field given by Lemma \ref{lem:gen:gen} yields 
    \begin{align}
        \label{eq:x:ve}
        \dot X_t= \frac{-d(1-\tau_t)\dot\tau_t+\sigma^2\tau_t\dot\tau_t}
        {d(1-\tau_t)^{2}+\sigma^2\tau_t^2}X_{t} 
        + \frac{d(1-\tau_t)\dot\tau_t}{d(1-\tau_t)^{2}+\sigma^2\tau_t^2} r\tanh\left(h+\frac{\tau_{t}r\cdot X_{t}}{d(1-\tau_t)^{2}+\sigma^2\tau_t^2}\right)    
    \end{align}
    \begin{align}
        \label{eq:m:ve}
        \dot M_t= \frac{-d(1-\tau_t)\dot\tau_t+\sigma^2\tau_t\dot\tau_t}
        {d(1-\tau_t)^{2}+\sigma^2\tau_t^2}M_{t} 
        + \frac{d(1-\tau_t)\dot\tau_t}{d(1-\tau_t)^{2}+\sigma^2\tau_t^2} \tanh\left(h+\frac{\tau_{t} dM_{t}}{d(1-\tau_t)^{2}+\sigma^2\tau_t^2}\right)    
    \end{align}
    where $X_0\sim \mathcal{N}(0, d\text{Id})$ and $M_0\sim \mathcal{N}(0, 1).$\\
    \\
    \textbf{First phase.} We consider $t\in[0,1/2],$ where $\tau_t=(1-\kappa/\sqrt{d})2t$ gives
    \begin{equation}
        \dot{M}_{t} = \frac{-M_t + \tanh\left(h + \frac{2t M_t}{(1-2t)^2}\right)}{\tfrac{1}{2}-t} + O\left(\frac{1}{\sqrt{d}}\right)
    \end{equation} 
    We hence get a well-defined equation for the magnetization. In fact, by reparameterizing time $t(s)=s/2$ with $t:[0,1]\to[0,1/2],$ we get from Lemma \ref{lem:gen:gen} that this the $1$-dimensional velocity field that transports ${\mathcal{N}}(0,1)$ at $t=0$ to $p\delta_1 + (1-p)\delta_{-1}$ at $t=1/2.$
    
    We let $X^\perp = X^{\Delta t}_t - \frac{r\cdot X^{\Delta t}_t}{{d}}r$ and note that equation \eqref{eq:x:ve} gives 
    \begin{align}
        \label{eq:x:ve:perp}
        \dot X^\perp_t= \frac{-d(1-\tau_t)\dot\tau_t+\sigma^2\tau_t\dot\tau_t}
        {d(1-\tau_t)^{2}+\sigma^2\tau_t^2}X_{t}^\perp
    \end{align}
    Since this is a linear ODE with Gaussian initial condition, $X^\perp_t$ will be Gaussian for every $t,$ even for nonzero $\Delta t.$ Let us determine its covariance. We decompose $X^\perp_t$ as $X^\perp_t=\sqrt{d}X^1_t + X^0_t.$ Plugging in $\tau_t=(1-\kappa/\sqrt{d})2t$ into equation \eqref{eq:x:ve} gives
    \begin{align}
        \sqrt{d}\dot{X}^1_t + \dot{X}^0_t &= \frac{-d\left(1-2t+\tfrac{2\kappa t}{\sqrt{d}}\right)\left(1-\tfrac{\kappa}{\sqrt{d}}\right)2+\sigma^2\left(1-\tfrac{\kappa}{\sqrt{d}}\right)^24t}
        {\left(\sqrt{d}(1-2t)+2\kappa t\right)^{2}+\sigma^2\left(1-\tfrac{\kappa}{\sqrt{d}}\right)^24t^2}\left(\sqrt{d}{X}^1_t + {X}^0_t\right).
    \end{align}
    Taylor expanding the RHS in powers of $\sqrt{d}$ and matching terms of order $\sqrt{d}$ gives
    \begin{align}
        \dot{X}^1_t = \frac{-X^1_t}{\tfrac{1}{2}-t}
    \end{align}
    with $X_t^1 \sim \mathcal{N}(0,\text{Id}).$ This means that for $t\in[0,1/2]$ we have $X^1_t = X^1_{t=0} (1-2t).$ Matching terms of constant order 
    \begin{align}
        \dot{X}^0_t &= -\frac{1}{\tfrac{1}{2}-t}X^0_t + \frac{2\kappa}{(1-2t)^2}X^1_t\\
        &= \frac{-X^0_t+\kappa X^1_{t=0}}{\tfrac{1}{2}-t}.
    \end{align}
    From here we conclude that $X^0_{t=1/2}=\kappa X^1_{t=0}.$ We then have 
    \begin{align}
        X_t^{\perp}\sim \mathcal{N}\left(0, d\left(\hat \sigma^{\Delta t, d}_t\right)^2\text{Id}_{d-1}\right).
    \end{align}
    where $\lim_{\Delta t\to 0}\lim_{d \to\infty} \sigma^{\Delta t, d}_t= 1-2t.$
    \\
    \textbf{Second phase.} We now consider $t\in[1/2, 1].$ Using the definition of $\tau_t$, we get from equation \eqref{eq:x:ve} that for $d$ large,
    \begin{equation}
        \dot{X}_{t}  = -\frac{2\kappa^2(2-2t)}{\kappa^2(2-2t)^2+\sigma^2}X_t+\frac{2k^2(2-2t)}{\kappa^2(2-2t)^2+\sigma^2} r\sgn(M_t).
    \end{equation}
    In particular, 
    \begin{equation}
        \dot{M}_{t}  = -\frac{2\kappa^2(2-2t)}{\kappa^2(2-2t)^2+\sigma^2}M_t+\frac{2k^2(2-2t)}{\kappa^2(2-2t)^2+\sigma^2} \sgn(M_t).
    \end{equation} 
    Since $M_{t=1/2}$ is either $1$ or $-1,$ we see that it will remain constant in the second phase. On the other hand, we have
    \begin{equation}
        \dot{X}^\perp_{t}  = -\frac{2\kappa^2(2-2t)}{\kappa^2(2-2t)^2+\sigma^2}X^\perp_t.
    \end{equation}
    Solving explicitly gives for $t\in[1/2, 1]$
    \begin{align}
        {X}^\perp_{t} = {X}^\perp_{1/2} \sqrt{\frac{\kappa^2(2-2t)^2+\sigma^2}{\kappa^2+\sigma^2}}.
    \end{align}
    From the analysis of the first phase, we know that 
    \begin{align}
        X^\perp_{1/2} \sim \mathcal{N}\left(0, \left(\hat \sigma^{\Delta t, d}_t\right)^2\text{Id}_{d-1}\right).
    \end{align}
    where $\lim_{\Delta t\to 0}\lim_{d \to\infty} \hat \sigma^{\Delta t, d}_t = \kappa.$ We hence get for $t\in[1/2, 1]$ that 
    \begin{align}
        X^\perp_{t} \sim \mathcal{N}\left(0, \left(\hat \sigma^{\Delta t, d}_t\right)^2\text{Id}_{d-1}\right).
    \end{align}
    where $\lim_{\Delta t\to 0}\lim_{d \to\infty} \hat \sigma^{\Delta t, d}_t = \kappa \sqrt{\frac{\kappa^2(2-2t)^2+\sigma^2}{\kappa^2+\sigma^2}}.$
\end{proof}

\section{Dilated interpolants capture both phases for CW model}
In this section, we prove Theorem \ref{thm:cw:ve} from the main text, showing that the dilated VE interpolant captures both phases for the CW distribution. We also state and prove Theorem \ref{prp:char:cw}, proving that the VP interpolant captures both phases. We will need the following lemma
\begin{lem}
    \label{lem:gen:den}
    For all $t\in [0,1]$, the law of the interpolant ${I}_{t}=c\alpha_t z + \beta_t a$ is the same as the law of ${X}_t,$ the solution to the probability flow ODE
    \begin{equation}   
        \label{eq:gen:y}
          \dot {{X}}_t = \frac{\dot\alpha_t}{\alpha_t} X_t + \frac{\alpha_t\dot\beta_t-\dot\alpha_t\beta_t}{\alpha_t}\mathbb{E}[a|{I}_t={X}_t], \qquad {X}_{t=0} \sim \mathcal{N}(0, c^2\text{Id})
    \end{equation}
\end{lem}
\begin{proof}
    This follows from combining the equations
    \begin{align}
        X_t = \alpha_t\mathbb{E}[z|{I}_t={X}_t] + \beta_t\mathbb{E}[a|{I}_t={X}_t]\\
        b_t(X_t) = \dot\alpha_t\mathbb{E}[z|{I}_t={X}_t] + \dot\beta_t\mathbb{E}[a|{I}_t={X}_t]
    \end{align}
    where the second equation follows from Theorem 2.6 in \cite{albergo2023stochasticinterpolantsunifyingframework}
\end{proof}
\begin{proof}[Proof of Theorem \ref{thm:cw:ve}]
    We write $\eta_t(x):=\mathbb{E}[a|I^\text{E}_t=x]$ explicitly as
    \begin{equation}
        \eta^i_t(x) =\frac{pQ_{+}(x)\tanh\left(\beta m+\frac{\tau_t x_{i}}{d(1-\tau_t)^2}\right)+(1-p)Q_{-}(x)\tanh\left(-\beta m+\frac{\tau_t x_{i}}{d(1-\tau_t)^2}\right)}{pQ_{+}(x)+(1-p)Q_{-}(x)} 
    \end{equation}
    where
    \begin{equation}
        \label{eq:q:def}
        Q_{\pm}(x)=\prod_{i=1}^d\left[1\pm m\tanh\left(\frac{\tau_t x_{i}}{d(1-\tau_t)^2}\right)\right].
    \end{equation}

\textbf{First phase.} 
    For $t \in [0, 1/2],$ we have 
    \begin{align}
        \tanh\left(\beta m+\frac{\tau_t x_{i}}{d(1-\tau_t)^2}\right)&\approx \tanh(\beta m) = m \\
        Q_\pm(x) &\approx \exp\left(\pm m \frac{\tau_t}{d(1-\tau_t)^2} r\cdot x \right)
    \end{align}
    where $r=(1,\cdots,1)$ and we linearized the $\tanh$ to get the second equation. These approximations require $\frac{\tau_t x_i}{d(1-\tau_t)^2}$ to be small. We note that
    \begin{align}
        \frac{\tau_t x_i}{d(1-\tau_t)^2}=\frac{\tau_t \left(\sqrt{d}(1-\tau_t)z+\tau_t a\right)}{d(1-\tau_t)^2} \approx \frac{z}{\sqrt{d}(1-\tau_t)}.
    \end{align}
    Since for $t\in[0,1/2]$ we have $\sqrt{d}(1-\tau_t)\geq \kappa$, then $\frac{\tau_t x_i}{d(1-\tau_t)^2} = O\left(\frac{1}{\kappa}\right).$ Hence $\kappa$ regulates how good these approximations are, so that for $\kappa$ large enough, they are valid.
    
    Combining the two approximations and using that $p=e^{hm}/(e^{hm}+e^{-hm})$ gives
    \begin{equation}
        \label{eta:2nd:phase}
        \eta_t(x) = r m \tanh\left(mh+m \frac{\tau_t}{d(1-\tau_t)^2}r\cdot x\right) +o_\kappa(1)
    \end{equation}
    Lemma \ref{lem:gen:den} tells us that the law of the interpolant $I^E_t$ is the same as that of $X_t,$ the solution to the ODE
    \begin{equation}   
        \label{eq:gen:a}
        \dot {{X}}_t = \frac{\dot\tau_t}{1-\tau_t}(-X_t + \eta_t(X_t)), \qquad {X}_{t=0} \sim \mathcal{N}(0, d\text{Id}).
    \end{equation}
    Putting this together, and using $M_t=\frac{r\cdot X_t}{d}$ gives 
    \begin{align}   
        \label{eq:gen:m}
        \dot {{M}}_t = \frac{\dot\tau_t}{1-\tau_t}\left(-M_t + m \tanh\left(mh+m \frac{\tau_t}{(1-\tau_t)^2}M_t \right) +o_\kappa(1)\right)
    \end{align}
    with ${M}_{t=0} \sim \mathcal{N}(0, 1).$ Using the definition of $\tau_t,$ we get
    \begin{equation}   
        \label{eq:gen:m_t}
        \dot {{M}}_t = \frac{-M_t + m \tanh\left(mh+\frac{2tmM_t}{(1-2t)^2}\right) +o_\kappa(1)}{\tfrac{1}{2}-t} + O\left(\frac{1}{\sqrt{d}}\right).
    \end{equation}
    We hence get a well-defined equation for the magnetization. In fact, after taking the limits $\kappa\to \infty,\Delta t\to0,d\to\infty,$ by reparameterizing time $t(s)=s/2$ with $t:[0,1]\to[0,1/2],$ we get from Lemma \ref{lem:gen:gen} that this the $1$-dimensional velocity field that transports ${\mathcal{N}}(0,1)$ at $t=0$ to $p\delta_m + (1-p)\delta_{-m}$ at $t=1/2$ as desired.

    Now let $X^\perp_t=X_t-\frac{r\cdot X_t}{d} r$ and write $X^\perp_t=\sqrt{d}X_t^1+X_t^0$
    \begin{align}
        \label{eq:ode:aa}
        \sqrt{d}\dot {X}^1_t+\dot {X}^0_t = \frac{1-\tfrac{\kappa}{\sqrt{d}}}{\tfrac{1}{2}-t+\tfrac{\kappa t}{\sqrt{d}}}\left(-\sqrt{d}{X}^1_t-{X}^0_t + o_\kappa(1)\right)
    \end{align}
    Taylor expanding the RHS in powers of $\sqrt{d}$ and matching terms of order $\sqrt{d}$ gives
    \begin{align}
        \dot {X}^1_t=\frac{-X^1_t}{\tfrac{1}{2}-t}
    \end{align}
    where $X_t^1 \sim \mathcal{N}(0,\text{Id}).$ Hence for $t\in[0,1/2]$ we have $X^1_t = X^1_{t=0}(1-2t).$ We now match terms of constant order in \eqref{eq:ode:aa} to get
    \begin{align}
        \dot{X}^0_t &= \frac{-X^0_t+o_\kappa(1)}{\tfrac{1}{2}-t} + \frac{2\kappa X^1_t}{(1-2t)^2}\\
        &= \frac{-X^0_t+\kappa X^1_{t=0}+o_\kappa(1)}{\tfrac{1}{2}-t}.
    \end{align}
    We get $X^0_{t=1/2}=\kappa X^1_{t=0}+o_\kappa(1)$. Fix $w\perp r,$ $|w|=1.$ Since $X^1_t = X^1_{t=0}(1-2t),$  we then have for $t\in[0,1/2]$ that
    \begin{align*}
        \lim_{\kappa\to\infty}\lim_{\Delta t\to 0}\lim_{d \to\infty} \tfrac{1}{\sqrt{d}}w\cdot (X_t-(1-2t)X_{0}) = 0.
    \end{align*} \\
    \textbf{Second phase.} Consider $t\in[1/2, 1].$ Using Lemma \ref{lem:gen:den} again, we get that the law of $I^E_t$ is the same as that of $X_t,$ which solves the ODE
    \begin{equation}   
        \label{eq:gen:a}
        \dot {{X}}_t = \frac{\dot\tau_t}{1-\tau_t}(-X_t + \eta_t(X_t)).
    \end{equation}
    Using the definition of $\tau_t$
    \begin{equation}   
        \label{eq:gen:aa}
        \dot {{X}}_t = \frac{1}{1-t}(-X_t + \eta_t(X_t))
    \end{equation}

    For $\kappa$ large enough, we can approximate $Q_\pm$ for $t=1/2$ as
    $$
    Q_\pm(x) \approx \exp\left(\pm m \frac{\tau_{1/2}}{(1-\tau_{1/2})^2} M_{1/2} \right).
    $$
    We note that $\frac{\tau_{1/2}}{(1-\tau_{1/2})^2}>\frac{d}{\kappa^2}$ and recall from the analysis of the first phase that under the appropiate limits we get $M_{1/2}\sim p\delta_1+(1-p)\delta_{-1}.$ In particular, we have that for $d$ large, either $Q_+ \gg Q_-$ or $Q_- \gg Q_+.$ We will approximate the ODEs under the assumption that this holds for $t>1/2,$ i.e., either $Q_+ \gg Q_-$ or $Q_- \gg Q_+$ for all $t>1/2.$ The resulting ODEs will allow us to compute the value for $Q_\pm$ and check that indeed either $Q_+ \gg Q_-$ or $Q_- \gg Q_+$ for $t>1/2$ showing self-consistency and justifying the assumption. We first note that our assumption on $Q_\pm$ means that we can approximate
    \begin{align}
        \eta_t(X_t) = \tanh\left(\beta m \sgn(M_t)r + \frac{1}{\kappa^2(2-2t)^2}X_t\right).
    \end{align}
    where $\tanh$ is applied elementwise. Combining this with \eqref{eq:gen:aa} gives
    \begin{align}
        \label{eq:x:cw:2nd:b}
        \dot X_t = \frac{-X_t+\tanh\left(\beta m \sgn(M_t)r + \frac{1}{\kappa^2(2-2t)^2}X_t\right)}{1-t} + O\left(\frac{1}{{d}}\right)
    \end{align}

    Fix a coordinate $i\in\{1,\cdots, d\}$ and without loss of generality, consider $\sgn(M_t)=1.$ We have 
    
    \begin{align}
        \label{eq:x:cw:2nd:b22}
        \dot X^i_t = \frac{-X^i_t+\tanh\left(\beta m + \frac{1}{\kappa^2(2-2t)^2}X^i_t\right)}{1-t} + O\left(\frac{1}{{d}}\right)
    \end{align}

    From our analysis of the first phase, we have that $X^i_{1/2} = kZ + m + o_\kappa(1)$ where $Z\sim \mathcal{N}(0,1).$ Under the change of variables $t(s)=s/2+1/2,$ equation \eqref{eq:x:cw:2nd:b22} becomes 
    \begin{align}
        \label{eq:x:cw:2nd:b3}
        \dot X^i_s = \frac{-X^i_s+\tanh\left(\beta m + \frac{1}{\kappa^2(1-s)^2}X^i_s\right)}{1-s} + O\left(\frac{1}{{d}}\right).
    \end{align}
    In the limit of $d\to\infty,$ we get from Lemma \ref{lem:g:to:sgmm2} that this velocity field transports $\kappa Z+m$ with $Z\sim\mathcal{N}(0,1)$ to $a\sim \left(\frac{1+m}{2}\right)\delta_1 + \left(\frac{1-m}{2}\right)\delta_{-1}.$ In particular, we know that for $t\in [1/2, 1],$ $X^i_t\stackrel{d}{=}\kappa (2-2t)Z+a.$ If we instead had $\sgn(M_{1/2})=-1,$ we would have the same results except that $a\sim \left(\frac{1-m}{2}\right)\delta_1 + \left(\frac{1+m}{2}\right)\delta_{-1}.$

    We will now argue that $M_t$ will remain fixed for $t \in [1/2, 1].$ Again without loss of generality, we take $\sgn(M_t) = 1$ and we get the evolution of the $X_t^i$ in equation \eqref{eq:x:cw:2nd:b2}. Since the $X^i_t$ are iid at $t=1/2$ and evolve identically and independently, we get that by the law of large numbers as $d\to\infty$
    $$M_t=\frac{1}{d}\sum_{i=1}^d X^i_t\to \mathbb{E}[X^1_t].$$
    Similarly, we get that as $d\to\infty$ 
    \begin{align}
        \label{eq:x:cw:2nd:b3}
        \dot M_t = \frac{-M_t+\mathbb{E}\left[\tanh\left(\beta m + \frac{1}{\kappa^2(2-2t)^2}X^1_t\right)\right]}{1-t}.
    \end{align}
    Hence, using that for $t\in [1/2, 1],$ $X^i_t\stackrel{d}{=}\kappa (2-2t)Z+a,$ with $\lambda =1/(\kappa(2-2t))$ we have
    \begin{align}
        \label{eq:x:cw:2nd:b4}
        \dot M_t = \frac{-M_t+\mathbb{E}\left[\tanh\left(\beta m + \lambda Z + \lambda^2 a\right)\right]}{1-t}.
    \end{align}
    We claim that $\mathbb{E}\left[\tanh\left(\beta m + \lambda Z + \lambda^2 a\right)\right] = m,$ from where it follows immediately that $M_t$ remains constant for $t\in[1/2,1].$ To prove the claim note that since $\mathbb{E}[a] = m,$ it suffices to show 
    \begin{equation}
        \label{eq:clm:on:b}
        \mathbb{E}\left[\tanh\left(\beta m + \lambda Z + \lambda^2 a\right) - a\right] = 0.
    \end{equation}
    We have 
    \begin{align}
        \mathbb{E}&\left[\tanh\left(\beta m + \lambda Z + \lambda^2 a\right) - a\right]\\ &= \left(\tfrac{1+m}{2}\right)\mathbb{E}\left[\tanh\left(\beta m + \lambda Z + \lambda^2 \right) - 1\right] + \left(\tfrac{1-m}{2}\right)\mathbb{E}\left[\tanh\left(\beta m + \lambda Z - \lambda^2 \right) + 1\right]\\
        &= \int dz e^{-z^2/2}\left(\left(\tfrac{1+m}{2}\right) \left(\tanh\left(\beta m + \lambda z + \lambda^2 \right) - 1\right) + \left(\tfrac{1-m}{2}\right) \left(\tanh\left(\beta m + \lambda z - \lambda^2 \right) + 1 \right) \right)\\
        &= \int dz e^{-(z^2+\lambda^2)/2}\left(e^{z\lambda}\left(\tfrac{1+m}{2}\right) \left(\tanh\left(\beta m + \lambda z \right) - 1\right) + e^{-z\lambda}\left(\tfrac{1-m}{2}\right) \left(\tanh\left(\beta m + \lambda z \right) + 1 \right) \right)
        \label{eq:tanh_exp}
    \end{align}
    where in the last equality we changed variables $z\to z-\lambda$ in the first term of the integral and $z\to z+\lambda$ in the second term.

    Since $\tanh(\beta m)=m$ yields
    $$
    \tanh(\beta m + \lambda z) = \frac{\left(\tfrac{1+m}{2}\right)e^{z\lambda }-\left(\tfrac{1-m}{2}\right)e^{-z\lambda}}{\left(\tfrac{1+m}{2}\right)e^{z\lambda }+\left(\tfrac{1-m}{2}\right)e^{-z\lambda}}
    $$
    we then get by rearranging that 
    $$e^{z\lambda}\left(\tfrac{1+m}{2}\right) \left(\tanh\left(\beta m + \lambda z \right) - 1\right) + e^{-z\lambda}\left(\tfrac{1-m}{2}\right) \left(\tanh\left(\beta m + \lambda z \right) + 1 \right) =0 $$ which implies that the integrand of equation \eqref{eq:tanh_exp} is zero, giving the desired result.

    We now check the self-consistency of the assumption that either $Q_+ \gg Q_-$ or $Q_- \gg Q_+.$ Using the definition of $Q_\pm$ and $\tau_t$ and the fact that the $X^i$ evolve independently and identically, we get by the law of large numbers that 
    \begin{align}
        \frac{1}{d}\log Q_\pm &= \frac{1}{d}\sum^d_{i=1}\log\left(1\pm m\tanh \left(\frac{X^i_t}{\kappa^2(2-2t)^2}\right)\right)\\
        &\to \mathbb{E}\left[\log\left(1\pm m\tanh \left(\frac{X^1_t}{\kappa^2(2-2t)^2}\right)\right)\right].
    \end{align}

    This means that to leading order in $d,$ we have $Q_\pm = \mathbb{E}[Q_\pm].$ We then have using the fact that $X^i_t$ evolve independently
    \begin{align}
        \mathbb{E}[Q_\pm] = \left(1\pm m\mathbb{E}\left[\tanh\left(\frac{X^1_t}{\kappa^2(2-2t)^2}\right)\right]\right)^d.
    \end{align}
    A similar computation to the one used to prove equation \eqref{eq:clm:on:b} gives 
    \begin{align}
        \mathbb{E}\left[\tanh\left(\frac{X^1_t}{\kappa^2(2-2t)^2}\right)\right] = m e^{-\lambda^2/2}\mathbb{E}\left[\frac{\sinh(\lambda Z)^2}{\cosh(\lambda Z)}\right] > 0.
    \end{align}
    where we evolve taking $\sgn(M_t)=1.$ This yields $Q_+ \gg Q_-$ as desired. If we had taken $\sgn(M_t)=-1,$ we would have gotten $Q_- \gg Q_+.$
\end{proof}

\begin{thm}[Dilated VP captures both features for CW model]
Let $X^{\Delta t}_t$ be obtained from the probability flow ODE associated with the dilated VP interpolant for the CW distribution discretized with a uniform grid with step size $\Delta t.$ Let $r=(1,\cdots,1)$ \\
    \textbf{First phase}: For $t\in [0,\tfrac{1}{2}],$ we have that
        \begin{align*}
            \mu_t = \lim_{\Delta t\to 0}\lim_{d \to\infty} \frac{r\cdot X^{\Delta t}_t}{\sqrt{d}}
        \end{align*}
        fulfills 
        \begin{align*}
            \dot \mu_t &= 2\kappa m \tanh\left(mh+2\kappa m t \mu_t\right),\quad \mu_{t=0}\sim \mathcal{N}(0,1).
        \end{align*}
        with $h$ such that $p=e^{mh}/(e^{mh}+e^{-mh}).$  This implies $\mu_{t=1/2}\sim p\mathcal{N}(\kappa m, 1) + (1-p)\mathcal{N}(-\kappa m, 1).$\\ 
  In addition, for $w\perp r,$ $|w|=1,$ we have for $t\in[0,1/2]$
    \begin{align*}
        \lim_{d \to\infty} \tfrac{1}{\sqrt{d}}w\cdot (X_t-X_{0}) = 0
    \end{align*}
    \textbf{Second phase}: For $t\in [\tfrac{1}{2}, 1],$ we have that
        \begin{align*}
            M_t = \lim_{\kappa\to\infty}\lim_{\Delta t\to 0}\lim_{d \to\infty}  \frac{r\cdot X^{\Delta t}_t}{{d}}
        \end{align*}
        fulfills, for $t\in (1/2, 1)$, the ODE 
        \begin{align*}
            \dot M_t=\frac{-M_t+m\sgn(\mu_{1/2})}{1-t}
        \end{align*}
        Moreover, for any coordinate $i$ we have that
    \begin{align*}
        X^i_t = \lim_{\kappa\to\infty} \lim_{\Delta t\to 0}\lim_{d\to\infty}(X^{\Delta t, d}_t)^i
    \end{align*}
    satisfies the ODE for $t\in [1/2, 1)$
    \begin{align*}
        \dot X_t^i = \frac{-X_t^i+\tanh\left(\beta m\sgn(M_t)+ \frac{2t-1}{(2-2t)^2} X^i_t\right)}{1-t}
    \end{align*}
    with the intial condition $X^i_{1/2}\sim \mathcal{N}(0, 1).$ This equation implies that 
    \begin{align*}
        X_1^i \sim \left(\tfrac{1+m\sgn(M_1)}{2}\right)\delta_1+\left(\tfrac{1-m\sgn(M_1)}{2}\right)\delta_{-1}.
    \end{align*}
    
    \label{prp:char:cw}
\end{thm}

\begin{proof}[Proof of Theorem 4]
    Following \cite{bm}, we write $\eta_t(x):=\mathbb{E}[a|I^P_t=x]$ explicitly as
    \begin{equation}
        \label{eq:rsc:eta}
        \eta^i_t(x) =\frac{pQ_{+}(x)\tanh\left(\beta m+\frac{\tau_t x_{i}}{(1-\tau_t)^2}\right)+(1-p)Q_{-}(x)\tanh\left(-\beta m+\frac{\tau_t x_{i}}{(1-\tau_t)^2}\right)}{pQ_{+}(x)+(1-p)Q_{-}(x)} 
    \end{equation}
    where
    \begin{equation}
        \label{eq:Qpm}
        Q_{\pm}(x)=\prod_{i=1}^d\left[1\pm m\tanh\left(\frac{\tau_t x_{i}}{(1-\tau_t)^2}\right)\right].
    \end{equation}

\textbf{First phase.} For $t \in [0, 1/2],$ we have 
\begin{align}
    \label{eq:eta:approx}
    \tanh\left(\beta m+\frac{\tau_t x_{i}}{(1-\tau_t)^2}\right)&\approx \tanh(\beta m) = m \\
    \label{eq:Q:approx}
    Q_\pm(x) &\approx \exp\left(\pm m \frac{\tau_t}{(1-\tau_t)^2} r\cdot x \right)
\end{align}
where $r=(1,\cdots,1)$ and we linearized the $\tanh$ to get the second equation. These approximations are valid since in the first phase $\frac{\tau_t x_i}{(1-\tau_t)^2}=O\left(\frac{1}{\sqrt{d}}\right)$ is small. Combining the two approximations gives
\begin{equation}
    \label{eta:1st:phase}
    \eta_t(x) \approx r m \tanh\left(mh+m \frac{\tau_t}{(1-\tau_t)^2}r\cdot x\right)
\end{equation}
Lemma \ref{lem:gen:den} tells us that the law of the interpolant $I^P_t$ is the same as that of $X_t,$ the solution to the ODE
\begin{equation}   
    \label{eq:gen:y}
      \dot {{X}}_t = \frac{\dot\tau_t}{1-\tau_t}(-X_t + \eta_t(X_t)), \qquad {X}_{t=0} \sim \mathcal{N}(0, \text{Id}).
\end{equation}
Combining the last two equations with $\tau_t=2\kappa t/\sqrt{d}$ gives us the ODE
\begin{equation}
    \label{eq:ode:yt}
    \dot {{X}}_t = \frac{2\kappa}{\sqrt{d}} \left(-X_t+rm\tanh\left(mh+2kmt\frac{r\cdot X_t}{\sqrt{d}}\right) \right)+ O\left(\frac{1}{d}\right), \qquad {X}_{t=0} \sim \mathcal{N}(0, \text{Id})
\end{equation}
Writing $\mu_t = r\cdot X_t /\sqrt{d}$ gives the induced equation
\begin{equation}
    \dot \mu_t = 2\kappa m\tanh\left(mh+2kmt\mu_t\right) + O\left(\frac{1}{\sqrt{d}}\right), \qquad {\mu}_{t=0} \sim \mathcal{N}(0, 1)
\end{equation}
  Taking $d\to\infty$ yields the limiting ODE for the $\mu_t$ in the first phase. By reparameterizing time $t(s)=s/2$ with $t:[0,1]\to[0,1/2],$ we get from Lemma \ref{lem:g:to:sgmm} that this the $1$-dimensional velocity field associated to the interpolant $I_s=\sqrt{1-s^2}z+sa$ that transports $z\sim{\mathcal{N}}(0,1)$ at $t(s=0)=0$ to $a\sim p{\mathcal{N}}(\kappa m,1) + (1-p){\mathcal{N}}(-\kappa m, 1)$ at $t(s=1)=1/2.$
  Now fix $w\perp r$ with $|w|=1.$ Let $\nu_t=w\cdot X_t/\sqrt{d}.$ From equation \eqref{eq:ode:yt}, we get that for $t\in[0,1/2]$
  $$
  \dot \nu_t = O\left(\frac{1}{\sqrt{d}}\right)
  $$
  This means that $\lim_{d\to\infty} \nu_t-\nu_0 = 0$ as claimed. 

\textbf{Second phase.} For $t\in [1/2, 1],$ we have using Lemma \ref{lem:gen:den} and the definition of $\tau_t$
\begin{align}
    \label{eq:x:cw:2nda}
    \dot X_t = \frac{-X_t+\eta_t(X_t)}{1-t}.
\end{align}
We will approximate $\eta_t(x)$ based on the fact that either $Q_+ \gg Q_-$ or $Q_- \gg Q_+.$ To see this, write $a=smr + z$ where $p=\mathbb{P}(s=1) = 1-\mathbb{P}(s=-1)$ and $z\sim{\mathcal{N}}(0,\text{Id}_{d})$ and note that at $t=1/2$ we have 
$$
\mu_{1/2} = \frac{r\cdot I_{1/2}}{\sqrt{d}} \stackrel{d}{=} Z + \kappa ms + O\left(\frac{1}{\sqrt{d}}\right)
$$
where $Z\sim \mathcal{N}(0,1).$ This means that for $\kappa$ large enough, we can approximate $Q_\pm$ as
$$
Q_\pm(x) \approx \exp\left(\pm m \frac{\tau_{1/2}}{(1-\tau_{1/2})^2} \sqrt{d}\mu_{1/2} \right) \approx \exp\left(\pm km^2s \frac{\sqrt{d}\tau_{1/2}}{(1-\tau_{1/2})^2} \right).
$$
We note that $\frac{\sqrt{d}\tau_{1/2}}{(1-\tau_{1/2})^2}>\kappa.$ In particular, we have that with probability that goes to $1$ as $\kappa$ goes to infinity, either $Q_+ \gg Q_-$ or $Q_- \gg Q_+.$ We will approximate the ODEs under the assumption that this holds for $t>1/2,$ i.e., either $Q_+ \gg Q_-$ or $Q_- \gg Q_+$ for all $t>1/2.$ Similarly to the proof of Theorem \ref{thm:cw:ve}, one can use the resulting equations to show self-consistency of this assumption. 

Our assumption on $Q_\pm$ implies that we can approximate
\begin{align}
    \eta_t(X_t) = \tanh\left(\beta m \sgn(M_t)r + \frac{2t-1}{(2-2t)^2}X_t\right).
\end{align}
where $\tanh$ is applied elementwise. Combining this with \eqref{eq:x:cw:2nda} gives
\begin{align}
    \label{eq:x:cw:2nd:b}
    \dot X_t = \frac{-X_t+\tanh\left(\beta m \sgn(M_t)r + \frac{2t-1}{(2-2t)^2}X_t\right)}{1-t} + O\left(\frac{1}{{d}}\right)
\end{align}
From our analysis of the first phase, we note that $X_{t=1/2}=Y+O(1/\sqrt{d})$ where $Y\sim \mathcal{N}(0,\text{Id}).$ Let us assume without loss of generality that $\sgn(\mu_{1/2})=1$ and fix a coordinate $i\in\{1,\cdots, d\}$
\begin{align}
    \label{eq:x:cw:2nd:b}
    \dot X^i_t = \frac{-X^i_t+\tanh\left(\beta m + \frac{2t-1}{(2-2t)^2}X^i_t\right)}{1-t} + O\left(\frac{1}{{d}}\right).
\end{align}
Under the change of variables $t(s)=s/2+1/2$ we get that the ODE becomes
\begin{align}
    \label{eq:x:cw:2nd:b2}
    \dot X^i_s = \frac{-X^i_s+\tanh\left(\beta m + \frac{s}{(1-s)^2}X^i_s\right)}{1-s} + O\left(\frac{1}{{d}}\right).
\end{align}
In the limit of $d\to\infty,$ we get from Lemma \ref{lem:g:to:sgmm2} that this velocity field transports $Z\sim\mathcal{N}(0,1)$ to $a\sim \left(\frac{1+m}{2}\right)\delta_1 + \left(\frac{1-m}{2}\right)\delta_{-1}$ using the interpolant $I_s=(1-s)Z+sa.$

From the equation for $X^i_t$ we deduce
\begin{align}
    \label{eq:x:cw:2nd:b}
    \dot M_t = \frac{-M_t+\mathbb{E}\left[\tanh\left(\beta m + \frac{2t-1}{(2-2t)^2}X^i_t\right)\right]}{1-t} + O\left(\frac{1}{{d}}\right).
\end{align}

A similar computation to the one in the proof of Theorem \ref{thm:cw:ve} yields $\mathbb{E}\left[\tanh\left(\beta m + \frac{2t-1}{(2-2t)^2}X^i_t\right)\right]=m,$ giving the desired equation for $M_t.$

\end{proof}

\section{Dilated VP $I_\tau=\sqrt{1-\tau^2}z+\tau a$ captures $p$ and $\sigma^2$}
We proved in Theorem 1 from the main text that taking the VP interpolant $I_\tau=(1-\tau)z+\tau a$ and using the dilation $\tau_t$ from equation (4) in the main text leads to correct estimation for both $p$ and $\sigma^2$ for the GM distribution. We now show that the same time dilation leads to correct estimation if we instead use the interpolant $I_\tau=\sqrt{1-\tau^2}z+\tau a.$ The analysis of the first phase mimics that of Theorem 1, since $\sqrt{1-\tau^2}\approx 1 \approx 1-\tau$ in the first phase. The second phases for these two interpolants are also similar, but the details of the ODEs change.
\begin{thm}
    Let $X^{\Delta t}_t$ be obtained from the probability flow ODE associated with the dilated VP interpolant $I_t=\sqrt{1-\tau_t^2}z+\tau_t a$ discretized with a uniform grid with step size $\Delta t.$ Then 
    \label{thm:dvp}
    \begin{align*}
        X^{\Delta t}_t - \frac{r\cdot X^{\Delta t}_t}{{d}}r\sim \mathcal{N}\left(0, \left(\hat \sigma^{\Delta t, d}_t\right)^2\text{Id}_{d-1}\right).
    \end{align*}
    where $\hat \sigma^{\Delta t, d}_t$ is characterized as follows:\\
    \textbf{First phase}: For $t\in [0,\tfrac{1}{2}]$ we have 
    \begin{align*}
        \lim_{\Delta t\to 0}\lim_{d \to\infty} \sigma^{\Delta t, d}_t =1. 
    \end{align*}
    In addition  
        \begin{align*}
            \mu_t = \lim_{\Delta t\to 0}\lim_{d \to\infty} \frac{r\cdot X^{\Delta t}_t}{\sqrt{d}}
        \end{align*}
        fulfills 
        \begin{align*}
            \dot \mu_t &= 2\kappa \tanh\left(h+2\kappa t \mu_t\right),\quad \mu_{t=0}\sim \mathcal{N}(0,1).
        \end{align*}
        where $h$ is such that $p=e^h/(e^h+e^{-h}).$ This implies $\mu_{t=1/2}\sim p\mathcal{N}(\kappa, 1) + (1-p)\mathcal{N}(-\kappa, 1).$\\ 
    \textbf{Second phase}: For $t\in [\tfrac{1}{2}, 1]$ we have 
    \begin{align*}
        \lim_{\Delta t\to 0}\lim_{d \to\infty} \sigma^{\Delta t, d}_t =\sqrt{1+(\sigma^2-1)(2t-1)^2}
    \end{align*}
    In addition 
        \begin{align*}
            M_t = \lim_{\Delta t\to 0}\lim_{d \to\infty}  \frac{r\cdot X^{\Delta t}_t}{{d}}
        \end{align*}
        fulfills, for $t\in (1/2, 1)$, the ODE 
        \begin{align*}
            \dot M_t=\frac{(\sigma^2-1)(2t-1)2}{1+(\sigma^2-1)(2t-1)^2}M_t + \frac{2r\sgn(M_t)}{1+(\sigma^2-1)(2t-1)^2}
        \end{align*}
        and  satisfies
        \begin{align*}
            M_1 = p^\kappa\delta_1 + (1-p^\kappa)\delta_{-1}
        \end{align*}
        where $p^\kappa$ is such that  $\lim_{\kappa\to \infty}p^\kappa=p$
\end{thm}
\begin{proof}
        Consider the dilated variance preserving interpolant $I_t = \sqrt{1-\tau_t^2}z + \tau_t a$ where $z\sim \mathcal{N}(0,\text{Id}),$ $a\sim \mu,$ and $\tau_t$ is given in equation (4) in the main text. Plugging in $\alpha_t=\sqrt{1-\tau_t}$ and $\beta_t=\tau_t$ into the velocity field given by Lemma \ref{lem:gen:gen} yields 
    \begin{align}
    \label{eq:gen:dil}
    \dot{X}_{t} = \frac{(\sigma^2-1)\tau_t\dot\tau_t}{1+(\sigma^2-1)\tau^2_t}X_t+\frac{\dot\tau_t}{1+(\sigma^2-1)\tau_t^2}  r\tanh\left(h+\frac{\tau_{t} r\cdot X_{t}}{1+(\sigma^2-1)\tau_t^2}\right)
    \end{align}
    \textbf{First phase.} For $t\in [0,1/2],$ we have $\tau_t = \frac{2 \kappa t}{\sqrt{d}}.$ Plugging in into equation \eqref{eq:gen:dil} gives
    \begin{align}
        \dot{X}_{t} & = \frac{2\kappa}{\sqrt{d}} r\tanh\left(h+2\kappa t \frac{r\cdot X_{t}}{\sqrt{d}}\right) + O\left(\frac{1}{d}\right).
    \end{align}
    The remaining of the analysis of the first phase to yield the desired results is almost identical to what we did in the proof of Theorem 1 and is omitted.\\
    \textbf{Second phase.} For $t\in[1/2,1],$ we have $\tau_t = \left(1-\frac{\kappa}{\sqrt{d}}\right)(2t-1)+\frac{\kappa}{\sqrt{d}}.$ Using equation \eqref{eq:gen:dil} again yields 
    \begin{align}
        \dot X_t = \frac{(\sigma^2-1)(2t-1)2}{1+(\sigma^2-1)(2t-1)^2}X_t + \frac{2r\tanh\left(h + \frac{(2t-1)r\cdot X_t + \kappa \frac{r\cdot X_t}{\sqrt{d}}}{1+(\sigma^2-1)(2t-1)^2}\right)}{1+(\sigma^2-1)(2t-1)^2} + O\left(\frac{1}{d}\right).
        \label{eq:ode:x:2}
    \end{align}
    Writing $\mu_t=\frac{r\cdot X_t}{\sqrt{d}},$ this implies
    \begin{align}
        \dot \mu_t = \frac{(\sigma^2-1)(2t-1)2}{1+(\sigma^2-1)(2t-1)^2}\mu_t + \frac{2\sqrt{d}\tanh\left(h + \frac{(2t-1)\sqrt{d}\mu_t + \kappa \mu_t}{1+(\sigma^2-1)(2t-1)^2}\right)}{1+(\sigma^2-1)(2t-1)^2} + O\left(\frac{1}{\sqrt{d}}\right).
        \label{eq:ode:mu:2}
    \end{align}
    From the analysis of the first phase (see equation \eqref{eq:vp:1st:mu} in the proof of Theorem 1), we have that for finite $d$ and discretizing with step size $\Delta t$ 
    \begin{align}
        \mu_{t=1/2} = \theta + O\left(\frac{1}{\sqrt{d}}\right) + o_{\Delta t}(1)
    \end{align}
    where $\theta \sim p\mathcal{N}(\kappa, 1) + (1-p)\mathcal{N}(-\kappa, 1)$ and the term $o_{\Delta t}(1)$ goes to zero as $\Delta t$ goes to zero independently of $d,$ since this error only comes from discretizing the $d$-independent ODE $\dot{\mu}_{t} = 2\kappa\tanh\left(h+2\kappa t \mu_t\right)$ with $\mu_{t=0}\sim \mathcal{N}(0,1).$

    At $t=1/2,$ the argument of the $\tanh$ is $h+\kappa\mu_{1/2}.$ Assume $\theta$ takes value on the $+\kappa$ mode. For $d$ large enough and $\Delta t$ small enough (independently of $d$) we have that $|\mu_{1/2} - \theta| < 1$. We also have that $h\ll \kappa \theta$ and hence $h\ll \kappa\mu_t,$ where both inequalities hold with probability going to $1$ as $\kappa$ goes to infinity. This means we can approximate the ODE for $\mu_t$ for $t=1/2$ as 
    \begin{align}
        \dot \mu_t = \frac{(\sigma^2-1)(2t-1)2}{1+(\sigma^2-1)(2t-1)^2}\mu_t + \frac{2\sqrt{d}\sgn(\mu_t)}{1+(\sigma^2-1)(2t-1)^2} + O\left(\frac{1}{\sqrt{d}}\right).
        \label{eq:ode:mu:3}
    \end{align}
    We note that this remains valid for $t>1/2$ since under the approximation we used in equation \eqref{eq:ode:mu:3}, we have that $\mu_t$ is increasing. Indeed, whenever $\mu_t=o(\sqrt{d}),$ the second term in the RHS of \eqref{eq:ode:mu:3} will dominate. If $b$ takes value on the $-\kappa$ mode instead, an analogous argument shows that \eqref{eq:ode:mu:3} is also valid in that case.
    
    We then use this approximation in the ODEs for $X_t$ to get for $t\in(1/2, 1)$
    \begin{align}
        \label{eq:final:x_t:2}
        \dot X_t=\frac{(\sigma^2-1)(2t-1)2}{1+(\sigma^2-1)(2t-1)^2}X_t + \frac{2r\sgn(M_t)}{1+(\sigma^2-1)(2t-1)^2} + O\left(\frac{1}{d}\right).
    \end{align}
    where $M_t=r\cdot X_t/d.$ This yields the limiting ODE for $M_t$ in the theorem statement. We recall that from the analysis of the first phase (after taking the limit first on $d\to\infty$ and then on $\Delta t\to 0$) we got 
    \begin{align}
        \mu_{t=1/2}\sim p\mathcal{N}(\kappa, 1) + (1-p)\mathcal{N}(-\kappa, 1).
    \end{align}
    We argued above that the sign of $\mu_t$ will be preserved for $t\in[1/2, 1]$ with probability going to $1$ as $\kappa$ tends to $\infty.$ This means that 
    \begin{align}
        M_1 = p^\kappa\delta_1 + (1-p^\kappa)\delta_{-1}
    \end{align}
    where $p^\kappa$ is such that  $\lim_{\kappa\to \infty}p^\kappa=p.$ 
    
    Let $X^\perp = X^{\Delta t}_t - \frac{r\cdot X^{\Delta t}_t}{{d}}r$ and note that
    \begin{align}
        \label{eq:final:x_t23}
        \dot X^{\perp}_t= \frac{(\sigma^2-1)(2t-1)2}{1+(\sigma^2-1)(2t-1)^2}X^{\perp}_{t}.
    \end{align}
    Since this is a linear ODE from a Gaussian initial condition, we have
    \begin{align}
    X^\perp_t \sim \mathcal{N}\left(0, \left(\hat \sigma^{\Delta t, d}_t\right)^2\text{Id}_{d-1}\right).
    \end{align}
    Under the change of variables $t(s)=s/2+1/2,$ the equation \eqref{eq:final:x_t23} for $X_t^\perp$ becomes
    \begin{align}
        \label{eq:final:x_t4}
        \dot X^{\perp}_s= \frac{-s+\sigma^2s}
        {(1-s^2)+\sigma^2s^2}X^{\perp}_{s}.
    \end{align}
    By taking one coordinate $i\in\{1,\cdots, d-1\}$ of $X_s^\perp$ we get from Lemma \ref{lem:g:to:g} that this is the velocity field associated with the interpolant $I_s=\sqrt{1-s^2}z+sa$ where $z\sim \mathcal{N}(0,1)$ is transported to $a\sim \mathcal{N}(0, \sigma^2).$ For fixed $s\in [0, 1],$ the interpolant $I_s$ has variance $(1-s^2)+\sigma^2s^2 = 1+(\sigma^2-1)(2t-1)^2$ as claimed.
\end{proof}

\section{Connection with the sub-VP SDE.}
Similarly to the connection described in Section 3.3 between our VP interpolant and the VP SDE from \cite{song2021scorebasedgenerativemodelingstochastic}, we note that the sub-VP SDE from \cite{song2021scorebasedgenerativemodelingstochastic} corresponds to the interpolant $I_t = (1-\tau_t)z + \tau_t a$ with $\tau_t=exp(-\int_0^{-\ln t}\gamma_u du)$ as in Section 3.3. Hence, the sub-VP SDE corresponds exactly to a time-dilation of the VP interpolant we analyze in Theorem 1.

\section{Details for the CelebA experiment}
As mentioned in the main text, we use pretrained models for the VP and VE SDEs from \cite{song2021scorebasedgenerativemodelingstochastic}. We use pretrained models on the CelebA-HQ dataset \cite{karras2018progressivegrowinggansimproved}. These models are available publicly on the HuggingFace library for the VP SDE \cite{huggingface_ddpm_celebahq_256} and for the VE SDE \cite{huggingface_ncsnpp_celebahq_256}. We note that for the VP SDE we actually use the DDPM model from \cite{ho2020denoising} which was later shown to correspond to a particular discretization of Song et al's VP SDE (see Appendix E in \cite{song2021scorebasedgenerativemodelingstochastic}.)

We then generate samples running the VP or VE SDEs with different number of discretization steps with a uniform grid. For a given number of discretization steps, we generate $7,500$ samples and then use the DeepFace library from \cite{serengil2024lightface} to detect whether there is a face in the generated image. This measures the low-level feature of the image. For the high-level feature, when the DeepFace library does detect a face, it tries to predict the race of the generated face, giving one of the following $6$ races: Asian, Black, Indian, Latino/Hispanic, Middle Eastern, White. Given the predicted races in the samples were a face was detected among the $7,500$, we calculate an empirical distribution supported on $6$ points. We also calculated the race distribution on the original CelebA-HQ dataset which has $30,000$ real images. We then compute the KL Divergence between the distribution on races of the generated images and the images in the dataset.

A technical detail is that the DDPM implementation from \cite{huggingface_ddpm_celebahq_256} can only handle a number of discretization steps that is of the form $\left \lfloor{1000/n}\right \rfloor $ where $n$ is an integer. For the VP SDE, we use one of the following options $\{10, 25, 50, 100, 250, 333, 500, 1000\}$ for the number of discretization steps. For the VE SDE, we instead use one of $\{333, 500, 750, 1000\}.$ A smaller number of discretization steps for the VE SDE leads to images that are too low-quality for our purposes.

As a sanity check, we include non-cherry-picked samples from the VP SDE in Figure \ref{fig:vp_sde} and from the VE SDE in Figure \ref{fig:ve_sde}. We confirm that diversity increases for the images generated by the VP SDE as the number of steps grows, whereas quality increases for the images generated by the VE SDE as we take larger number of steps.
\begin{figure}
    \centering
    \includegraphics[width=\linewidth]{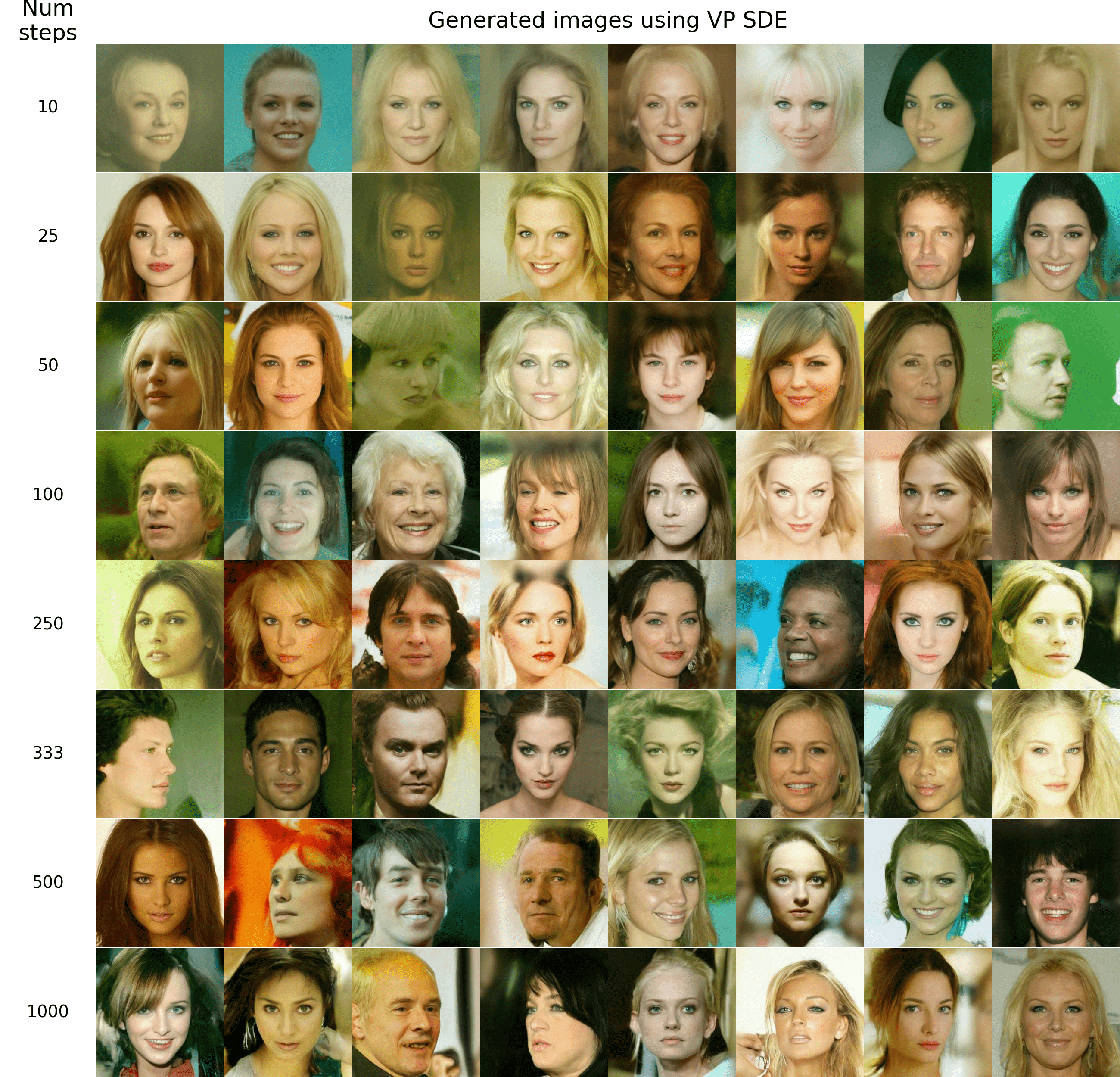}
    \caption{For different number of discretization steps, we include images generated by the VP SDE from \cite{song2021scorebasedgenerativemodelingstochastic} pretrained on the CelebA-HQ dataset \cite{huggingface_ddpm_celebahq_256}. We see that for small number of steps, the samples look alike, and diversity increases with the number of steps.}
    \label{fig:vp_sde}
\end{figure}

\begin{figure}
    \centering
    \includegraphics[width=1\linewidth]{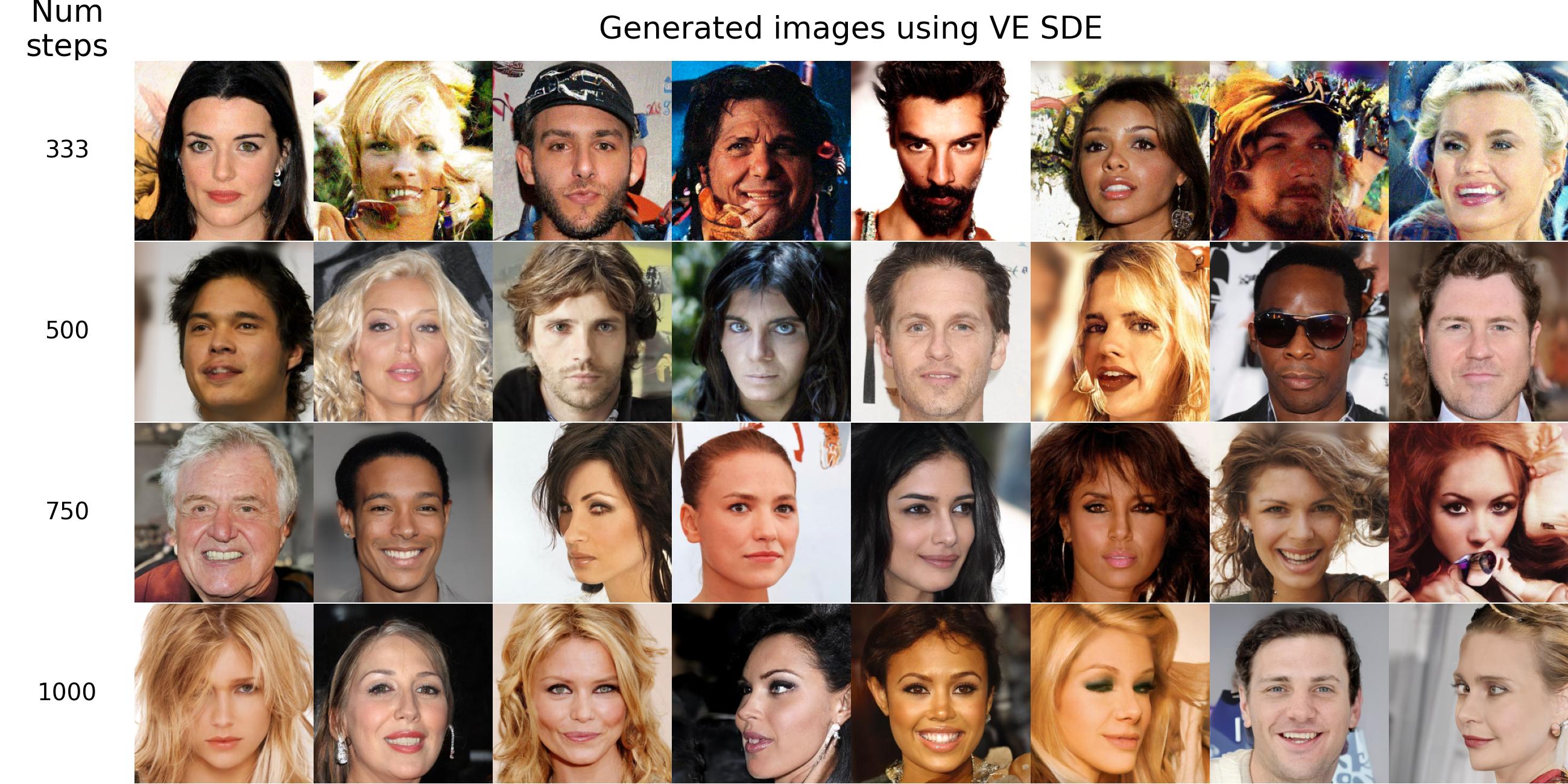}
    \caption{or different number of discretization steps, we show images generated by the VE SDE from \cite{song2021scorebasedgenerativemodelingstochastic} pretrained on the CelebA-HQ dataset \cite{huggingface_ncsnpp_celebahq_256}. Samples with small number of steps are lacking in quality, but not in diversity. As we increase the number of steps, the quality improves.}
    \label{fig:ve_sde}
\end{figure}

\end{document}